\documentclass[article]{my}
\newcommand{\paragraph}[1]{\textit{#1}.~}

\newcommand{\myemph}[1]{#1} 

\usepackage[OT1]{fontenc}
\usepackage[utf8]{inputenc}
\usepackage{color}
\usepackage{amsmath}
\usepackage{amssymb,latexsym}
\usepackage[font={small}]{caption}
\usepackage{graphicx}
\usepackage{etoolbox,xparse,xspace}
\usepackage{array,multirow,booktabs} 
\usepackage{bigstrut} 
\usepackage{bbm} 
\usepackage{dsfont} 

\usepackage[multiple]{footmisc} 

\usepackage{tikz,pgfplots,tikz-3dplot}
\usetikzlibrary{arrows.meta,shapes.misc,positioning,fit,calc,intersections,through,decorations.markings}
\tdplotsetmaincoords{55}{135}        

\tikzstyle{cell}=[dashed,thick]
\tikzstyle{simplex}=[thick]
\newcommand{\tikzfigscale}{2.2}
\newcommand{\placefiglabel}[1]{\node at (1.1,0,1.7) {#1};}

\graphicspath{{../prop-figs/}}

\usepackage[colorinlistoftodos,textsize=tiny]{todonotes} 

\newcommand{\Comments}{0}
\definecolor{gray}{gray}{0.5}
\definecolor{lightred}{rgb}{1,0.6,0.6}
\definecolor{darkgreen}{rgb}{0,0.5,0}
\newcommand{\mynote}[2]{\ifnum\Comments=1\textcolor{#1}{#2}\fi}
\newcommand{\mytodo}[2]{\ifnum\Comments=1
  \todo[linecolor=#1!80!black,backgroundcolor=#1,bordercolor=#1!80!black]{#2}\fi}

\newcommand{\relint}{\textnormal{relint}}
\newcommand{\interior}{\mathrm{int}\,}
\newcommand{\tr}{\top}
\newcommand{\id}{\mathrm{id}}
\newcommand{\Var}{\mathrm{Var}}

\newcommand{\A}{\mathcal{A}}
\newcommand{\C}{\mathcal{C}}
\newcommand{\D}{\mathcal{D}}
\newcommand{\E}{\mathbb{E}}
\renewcommand{\H}{\mathcal{H}}
\newcommand{\I}{\mathcal{I}}
\newcommand{\fin}{{\mathrm{fin}}}
\newcommand{\N}{\mathbb{N}}
\renewcommand{\P}{\mathcal{P}}
\newcommand{\gaussians}{\mathcal{G}_\mathrm{mix}}
\newcommand{\Pquant}{\mathcal{P}_\mathrm{q}}
\newcommand{\Q}{\mathbb{Q}}
\newcommand{\Qc}{\mathcal{Q}}
\newcommand{\R}{\mathcal{R}}
\let\oldS\S 
\renewcommand{\S}{\mbox{\oldS\hspace{-0.5mm}}}
\newcommand{\V}{\mathcal{V}}

\newcommand{\X}{\mathcal{X}}
\newcommand{\Y}{\mathcal{Y}}

\newcommand{\convhull}{\mathrm{conv}}
\newcommand{\conv}{\convhull}

\newcommand{\countinf}{\infty}

\newcommand{\defeq}{\doteq}
\newcommand{\ones}{\mathds{1}}  

\newcommand{\sgn}{\mathrm{sgn}}
\newcommand{\im}{\mathop{\mathrm{im}}}
\newcommand{\spn}{\mathop{\mathrm{span}}}

\def\reals{\mathbb{R}}

\newcommand{\argmin}{\mathop{\mathrm{argmin}}}
\newcommand{\argmax}{\mathop{\mathrm{argmax}}}
\newcommand{\arginf}{\mathop{\mathrm{arginf}}}

\newcommand{\inprod}[1]{\left\langle #1 \right\rangle}

\newcommand{\elic}{\mathsf{elic}}
\newcommand{\elici}{\elic_\ID}
\newcommand{\elicifin}{\elic_{\I^\fin}}
\newcommand{\iden}{\mathsf{iden}}
\newcommand{\idprop}{\Gamma_{\id}}
\newcommand{\EL}{\mathcal{E}}
\newcommand{\ID}{\mathcal{I}}
\newcommand{\ES}{\mathrm{ES}}
\newcommand{\lbar}{\underline{L}}

\newcommand{\affdim}{\mathrm{affdim}}
\newcommand{\codim}{\mathrm{codim}}

 \renewcommand{\S}{\mbox{\oldS\hspace{-0.5mm}}}
\newcommand{\Clin}{{\C_\mathrm{lin}}}
\newcommand{\Ccvx}{{\C_\mathrm{cvx}}}
\newcommand{\Cstrict}{{\C_\mathrm{strict}}}
\newcommand{\Cstrong}{{\C_\mathrm{strong}}}
\newcommand{\mode}{\mathrm{mode}}
\newcommand{\mi}{\mathrm{mi}}
\renewcommand{\ones}{\mathds{1}} 

\newcounter{foo}
\newcommand{\tempsetcounter}[2]{
  \setcounter{foo}{\value{#1}}
  \setcounter{#1}{#2}
}
\newcommand{\restorecounter}[1]{\setcounter{#1}{\value{foo}}}

\renewcommand{\Comments}{1}
\ifnum\Comments=1
\fi

\begin{document}

\jname{Biometrika}

\markboth{R. Frongillo \and I.A. Kash}{Elicitation Complexity of Statistical Properties}

\title{Elicitation Complexity of Statistical Properties}

\author{RAFAEL FRONGILLO}
\affil{Department of Computer Science, University of Colorado Boulder\\ 1111 Engineering Dr., Boulder, CO, USA 80309 \email{raf@colorado.edu}}

\author{\and IAN A. KASH}
\affil{Department of Computer Science, University of Illinois Chicago\\
  851 S. Morgan (M/C 152), Room 1120 SEO, Chicago, IL 60607-7053
  \email{iankash@uic.edu}}

\maketitle

\begin{abstract}
  A property, or statistical functional, is said to be elicitable if it minimizes expected loss for some loss function.
  The study of which properties are elicitable sheds light on the capabilities and limitations of point estimation and empirical risk minimization.
  While recent work asks which properties are elicitable, we instead advocate for a more nuanced question: how many dimensions are required to \myemph{indirectly} elicit a given property?
  This number is called the elicitation complexity of the property.
  We lay the foundation for a general theory of elicitation complexity, including several basic results about how elicitation complexity behaves, and the complexity of standard properties of interest.
  Building on this foundation, our main result gives tight complexity bounds for the broad class of Bayes risks.
  We apply these results to several properties of interest, including variance, entropy, norms, and several classes of financial risk measures.
  We conclude with discussion and open directions.
\end{abstract}


\section{Introduction}
\label{sec:elic-complex-introduction}

Loss functions are used throughout statistics and machine learning, in tasks ranging from estimation and model selection, to forecast ranking and comparison \citep{gneiting2007strictly,gneiting2011making}.
In particular, through the ubiquitous paradigm of empirical risk minimization,
a model is chosen to minimize a loss function, perhaps with regularization, averaged over a data set.
To understand the asymptotic behavior of empirical risk minimization, and to understand the design tradeoffs in choosing the loss function more broadly, we may ask what property the loss elicits.
Here a property is a functional assigning a value, or vector of values, to each distribution, and a loss elicits a property if for each distribution, the property value uniquely minimizes the expected loss.
The study of which properties are elicitable thus addresses which statistics are computable via empirical risk minimization \citep{steinwart2008support,steinwart2014elicitation,agarwal2015consistent,frongillo2015vector}.

The literature on property elicitation takes its roots in statistics \citep{savage1971elicitation,osband1985providing,gneiting2007strictly,gneiting2011making}, branching more recently into machine learning \citep{abernethy2012characterization,steinwart2014elicitation,agarwal2015consistent,frongillo2015vector}, economics \citep{lambert2018elicitation,lambert2009eliciting}, and finance \citep{emmer2015what,bellini2015elicitable,ziegel2016coherence,wang2015elicitable,fissler2016higher}.
A line of work initiated by Savage~\citeyearpar{savage1971elicitation} looks at questions of characterization: which losses elicit the mean of a distribution, or more generally the expectation of a vector-valued random variable \citep{banerjee2005optimality,frongillo2015vector}, and which real-valued properties are elicitable \citep{lambert2008eliciting,steinwart2014elicitation,lambert2018elicitation}.
Apart from special cases, the characterization of elicitable vector-valued properties remains open, with only partial progress \citep{frongillo2015vector,agarwal2015consistent,fissler2016higher,fissler2019order}.
A recent parallel thread of research in finance seeks to understand which financial risk measures, among several in use or proposed to help regulate the risks of financial institutions, are elicitable; cf.\ references above.
More often than not, these works conclude that risk measures are not elicitable \citep{gneiting2011making,wang2015elicitable,wang2018risk}, with notable exceptions being generalized quantiles, e.g., value-at-risk and expectiles, and expected utility \citep{ziegel2016coherence,bellini2015elicitable}.

All through the literature on property elicitation, one question is central: which properties are elicitable?
Yet it is clear that all properties are ``indirectly'' elicitable if one first elicits the entire distribution using a standard proper scoring rule \citep{gneiting2007strictly}.
Hence, if a statistical property is found not to be elicitable, such as the variance, rather than abandoning it one may ask how many dimensions are required to elicit it.
In the present work, we thus ask the more nuanced question: \myemph{how} elicitable are properties?
Specifically, we adapt and generalize the notion of \myemph{elicitation complexity} introduced by Lambert et al.~\citeyearpar{lambert2008eliciting}, which captures how many prediction dimensions one needs in empirical risk minimization for the property in question.
In particular, upper bounds on elicitation complexity often give statistically consistent surrogate losses for a given property of interest.  Both upper and lower bounds address the dimension of the range of the intermediate hypothesis needed for this indirect elicitation; see \S~\ref{sec:erm}.

Our main result gives tight bounds on elicitation complexity for a large class of risk measures.
This result is heavily inspired by recent work of Fissler and Ziegel~\citeyearpar{fissler2016higher}, showing that spectral risk measures of support $k$ have elicitation complexity at most $k+1$.
Spectral risk measures, which include conditional value at risk (CVaR), also known as expected shortfall, are among those under consideration in the finance community.
Their result shows that, while not elicitable in the classical sense, the elicitation complexity of spectral risk measures is still low, and hence one can develop reasonable regression and ``backtesting'' procedures for them \citep{fissler2016expected,rockafellar2018superquantile/cvar}.
Our results extend to these and many other risk measures (\S~\ref{sec:ex-es}--\ref{sec:ex-rockafellar}), often providing matching lower bounds on the complexity as well.
Other related work has appeared in machine learning, giving what could be considered bounds on elicitation complexity with respect to linear and convex-elicitable properties \citep{ramaswamy2013convex,agarwal2015consistent}; see \S~\ref{sec:classes-props},\,\S~\ref{sec:discussion}.

Our contributions are the following.
We introduce a general definition of elicitation complexity with respect to a given class of properties, which is flexible enough to capture previous definitions in the literature, yet brings several advantages (\S~\ref{sec:setting-elic-complex}; \S~\ref{sec:comp-other-defin}).
Our main result gives matching upper and lower bounds on elicitation complexity for the broad class of \myemph{Bayes risks}, the optimal expected loss as a function of the underlying distribution (\S~\ref{sec:main-result}).
We then apply this result to several settings of interest, including entropy and norms of distributions, financial risk measures, and empirical risk minimization (\S~\ref{sec:elic-complex-exampl-appl}).
We provide a foundation for the more general study of elicitation complexity by establishing bounds for several basic properties such as expectations and quantiles, as well as results on how elicitation complexity behaves with respect to various operations (\S~\ref{sec:basic-results}).
We then prove our main results (\S~\ref{sec:elic-complex-bayes-risk}) and discuss various open questions (\S~\ref{sec:discussion}).

\section{Setting and Main Result}
\label{sec:elic-complex-setting}

\subsection{Preliminaries}

Let $\Y$ be a set of outcomes and $\P$ be a convex set of probability measures on $\Y$.
See \S~\ref{sec:extens-non-conv} for when the convexity assumption can be lifted.
The goal of elicitation is to learn something about the distribution $p \in \P$, specifically some function or \myemph{property} $\Gamma(p)$ such as the mean or variance, by minimizing a loss function.
When $\Y = \reals^k$, we will assume the Borel $\sigma$-algebra, and when $\Y$ is generic, the $\sigma$-algebra will be left implicit, but the relevant functions need to be measurable and $\P$-integrable, i.e., integrable with respect to each $p \in \P$.
Throughout, we will use $Y$ as the random variable representing the outcome itself, i.e.\ $Y:\Y\to\Y$, $y\mapsto y$, leaving $X$ to refer to an arbitrary random variable.

\begin{remark}\label{remark:notation}
When $\Y=\reals$, it would be more natural in many cases to discuss properties of random variables of the form $Y:\Omega\to\Y$, such as $\Gamma(Y) = \E[Y]$, where now $\Omega$ is the outcome set endowed with some fixed base measure $\mu$, thus eliminating the need for $p$.
In most examples, such as all risk measures discussed in this paper, $\Gamma$ would depend on $Y$ only through its law, in which case it is also natural to design loss functions which depend only on $y = Y(\omega)$ rather than allowing them direct access to $\omega\in\Omega$.
Thus, without loss of generality we could define $\Gamma(p) \defeq \Gamma(Y)$ where $p$ is the law of $Y$, and let the outcome set again be $\Y$, and $Y$ be the identity map; e.g.\ $\Gamma(p) = \E_p[Y]$.
This transformation is the reasoning behind the notation in this paper.
\end{remark}

With notation in hand, we can now introduce our central object of study, a property.
\begin{definition}
  \label{def:property}
  Let $\R$ be a nonempty set of \myemph{reports}. 
  A \myemph{property} is a functional $\Gamma : \P \to \R$, which associates a desired report value to each distribution.
  The \myemph{level set} $\Gamma_r \defeq \{p\in \P \,|\, r=\Gamma(p)\}$ is the set of distributions $p$ corresponding to report value $r\in\R$.
  A \myemph{set-valued property} is a functional $\Gamma : \P \to 2^\R$, where $2^\R$ denotes the powerset of $\R$.
\end{definition}

Given a property $\Gamma$, we are interested in the existence of a \myemph{loss function} whose expectation under $p$ is minimized by $\Gamma(p)$.
A loss function can be thought of as incentivizing a risk-neutral agent to reveal the correct value of the property according to their private belief.

\begin{definition}
  \label{def:loss-elicits}
  A \myemph{loss function}, or simply \myemph{loss}, is a function $L:\R\times\Y\to\reals$ such that $L(r,\cdot)$ is $\P$-integrable for all $r\in\R$.
  A loss $L$ \myemph{elicits} a property $\Gamma:\P\to \R$ if for all $p\in\P$,
  $\{\Gamma(p)\} = \argmin_{r} L(r,p)$,
  where $L(r,p) \defeq \E_p[L(r,Y)]$.
  A property is \myemph{elicitable} if some loss elicits it.
  If we instead have $\Gamma(p) \in \argmin_{r} L(r,p)$ for all $p\in\P$, we say $L$ \myemph{weakly elicits} $\Gamma$.
\end{definition}
For example, when $\Y=\reals$, the mean $\Gamma(p) = \E_p[Y]$ is elicitable via squared loss $L(r,y) = (r-y)^2$, provided the relevant expectations are finite.
While a constant loss function weakly elicits every property, and thus weak elicitability is trivial, it can be useful to discuss the set of losses weakly eliciting a property, as in Theorem~\ref{thm:elic-minimum}.

When $\Gamma$ is set-valued, we say $L$ elicits $\Gamma$ if $\Gamma(p) = \argmin_r L(r,p)$, i.e., the set of minimizers of the expected loss is given by $\Gamma$ \citep{frongillo2015vector}.
For example, the median can be set-valued, such as for distributions with disconnected support, and is elicited by $L(r,y) = |r-y|$ in the above sense.
Rather than developing the notation needed to compose set-valued maps to define elicitation complexity for these general properties, we instead refer to set-valued properties only when needed, notably in Theorem~\ref{thm:elic-minimum} and \S~\ref{sec:ex-modal-mass}, and otherwise assume single-valued properties.

\subsection{Elicitation Complexity}
\label{sec:setting-elic-complex}

To motivate elicitation complexity, consider the well-known necessary condition for elicitability, that the level sets of the property be convex.
\begin{proposition}[Osband~\citeyearpar{osband1985providing}]
  \label{prop:elic-complex-level-sets-convex}
  If $\Gamma$ is elicitable, the level sets $\Gamma_r$ are convex for all $r\in\Gamma(\P)$.
\end{proposition}
\noindent
This condition is not sufficient; for example, the mode has convex level sets but is not elicitable \citep{heinrich2013mode}.
As illustrated in Figure~\ref{fig:mean-squared-variance-lev-sets}(L,R), while the mean $\Gamma(p) = \E_p[Y]$ has convex level sets, the variance $\Var(p) = \E_p[(Y-\E_p[Y])^2]$ does not, and hence is not elicitable \citep{osband1985providing,lambert2018elicitation}.
Note however that writing $\Var(p) = \E_p[Y^2] - \E_p[Y]^2$ suggests the following approach: first elicit the property $\hat\Gamma(p) = (\E_p[Y],\E_p[Y^2])$, and then use this information to compute $\Var(p)$.
It is well-known \citep{savage1971elicitation,gneiting2011making} that such a $\hat\Gamma$ is elicitable as the expectation of a vector-valued random variable $\phi(y)=(y,y^2)$, using for example $L(r,y) = \|r-\phi(y)\|_2^2$.

\begin{figure}[!t]
  \centering
  \begin{tikzpicture} [scale=\tikzfigscale, thick, tdplot_main_coords]
    \coordinate (orig) at (0,0,0);

    \pgfmathsetmacro{\ph}{0.54}
    \pgfmathsetmacro{\pm}{0.23}
    \pgfmathsetmacro{\pl}{0.23}
    \coordinate (p) at (\ph,\pm,\pl);
    \coordinate[label=below:${y=-1\!\!\!\!}$] (h) at (1,0,0);
    \coordinate[label=below:${y=1}$] (m) at (0,1,0);
    \coordinate[label=above:${y=0}$] (l) at (0,0,1);

    \draw[simplex] (h) -- (m) -- (l) -- (h);

    \begin{scope}
      \clip (h) -- (m) -- (l) -- (h);
      \foreach \c in {0,0.05,...,1}
      {
        \pgfmathsetmacro\aa{\c}
        \coordinate (A) at ($(\aa,1-\aa,0)$);
        \coordinate (B) at ($(\aa,1-\aa,2)$);
        \draw[blue] (A) -- (B);
      }
      \draw[orange, very thick] (0.3,0.7,0) -- (0.3,0.7,2);
      \draw[orange, very thick] (0.7,0.3,0) -- (0.7,0.3,2);
    \end{scope}
    \node[inner sep=2pt] (levset1) at (0.3,0.7,-.3) {$\Gamma_{0.4}\!\!$};
    \node[inner sep=2pt] (levset2) at (0.7,0.3,-.3) {$\Gamma_{\!\!-0.4}\!\!\!\!$};
    \draw[-latex,thin] (levset2) edge[in=270,out=90] (0.7,0.3,0);
    \draw[-latex,thin] (levset1) edge[in=270,out=90] (0.3,0.7,0);
    \placefiglabel{(L)}
  \end{tikzpicture}
  \hfill
  \begin{tikzpicture} [scale=\tikzfigscale, thick, tdplot_main_coords]
    \coordinate (orig) at (0,0,0);

    \pgfmathsetmacro{\ph}{0.54}
    \pgfmathsetmacro{\pm}{0.23}
    \pgfmathsetmacro{\pl}{0.23}
    \coordinate (p) at (\ph,\pm,\pl);
    \coordinate[label=below:${y=-1\!\!\!\!}$] (h) at (1,0,0);
    \coordinate[label=below:${y=1}$] (m) at (0,1,0);
    \coordinate[label=above:${y=0}$] (l) at (0,0,1);

    \draw[simplex] (h) -- (m) -- (l) -- (h);

    \begin{scope}
      \clip (h) -- (m) -- (l) -- (h);
      \foreach \c in {0,0.05,...,1}
ppp      {
        \pgfmathsetmacro\aa{\c}
        \coordinate (A) at ($(\aa,1-\aa,0)$);
        \coordinate (B) at ($(\aa,1-\aa,2)$);
        \draw[blue] (A) -- (B);
      }
      \draw[orange, very thick] (0.3,0.7,0) -- (0.3,0.7,2);
      \draw[orange, very thick] (0.7,0.3,0) -- (0.7,0.3,2);
    \end{scope}
    \node (levset) at (0.5,0.5,-.3) {$\Gamma_{0.16}$};
    \draw[-latex,thin] (levset) edge[in=270,out=180] (0.7,0.3,0);
    \draw[-latex,thin] (levset) edge[in=270,out=0] (0.3,0.7,0);
    \placefiglabel{(M)}
  \end{tikzpicture}
  \hfill
  \begin{tikzpicture} [scale=\tikzfigscale, thick, tdplot_main_coords]
    \coordinate (orig) at (0,0,0);

    \pgfmathsetmacro{\ph}{0.54}
    \pgfmathsetmacro{\pm}{0.23}
    \pgfmathsetmacro{\pl}{0.23}
    \coordinate (p) at (\ph,\pm,\pl);
    \coordinate[label=below:${y=-1\!\!\!\!}$] (h) at (1,0,0);
    \coordinate[label=below:${y=1}$] (m) at (0,1,0);
    \coordinate[label=above:${y=0}$] (l) at (0,0,1);

    \draw[simplex] (h) -- (m) -- (l) -- (h);

    \begin{scope}
      \clip (h) -- (m) -- (l) -- (h);
      \foreach \c in {0,0.05,...,1}
      {
        \pgfmathsetmacro\aa{0.5*(sqrt(1-\c)+1)}
        \pgfmathsetmacro\bb{0.5*\c}
        \coordinate (A) at ($(\aa,1-\aa,0)$);
        \coordinate (B) at ($(1-\aa,\aa,0)$);
        \coordinate (C) at ($(\bb,\bb,1-2*\bb)$);
        \draw[blue] (C) parabola (A);
        \draw[blue] (C) parabola (B);
      }
      \def\c{0.8}
      \pgfmathsetmacro\aa{0.5*(sqrt(1-\c)+1)}
      \pgfmathsetmacro\bb{0.5*\c}
      \coordinate (A) at ($(\aa,1-\aa,0)$);
      \coordinate (B) at ($(1-\aa,\aa,0)$);
      \coordinate (C) at ($(\bb,\bb,1-2*\bb)$);
      \draw[orange, very thick] (C) parabola (A);
      \draw[orange, very thick] (C) parabola (B);
    \end{scope}
    \def\c{0.8}
    \pgfmathsetmacro\aa{0.5*(sqrt(1-\c)+1)}
    \pgfmathsetmacro\bb{0.5*\c}
    \node (levset) at (0.5,0.5,-.3) {$\Gamma_{0.8}$};
    \draw[-latex,thin] (levset) edge[in=270,out=180] ($(\aa,1-\aa,0)$);
    \draw[-latex,thin] (levset) edge[in=270,out=0] ($(1-\aa,\aa,0)$);
    \placefiglabel{(R)}
  \end{tikzpicture}
  \caption{Level sets for the mean, squared mean, and variance.
    For each we use outcome space $\Y=\{-1,0,1\}$, and depict the probability simplex projected into two dimensions.
    Thus, the point distribution $p$ with $\Pr[Y=0]=1$ lies at the top point of the triangle, and the uniform distribution in the center.
    The squared mean and variance are not elicitable, as evidenced by their non-convex level sets.
    \\
    \textbf{(L)}
    The level sets for $\Gamma(p) = \E_p[Y]$.
    \quad
    \textbf{(M)}
    The level sets for $\Gamma(p) = (\E_p[Y])^2$.
    For $r > 0$ each level set $\Gamma_r = \{p : \Gamma(p)=r\}$ consists of two disjoint line segments, corresponding to the sets $\{p : \E_p[Y] = \sqrt{r}\}$ and $\{p : \E_p[Y] = -\sqrt{r}\}$.
    The natural link function $f(r) = r^2$ from the mean, so that $\Gamma(p) = (\E_p[Y])^2 = f(\E_p[Y])$, can be thought of as combining level sets of $\E[Y]$ to form the level sets of $\E[Y]^2$.
    \quad
    \textbf{(R)}
    The level sets for $\Gamma(p) = \Var(p) = \E_p[Y^2]-\E_p[Y]^2$, which are non-convex.
}
  \label{fig:mean-squared-variance-lev-sets}
\end{figure}
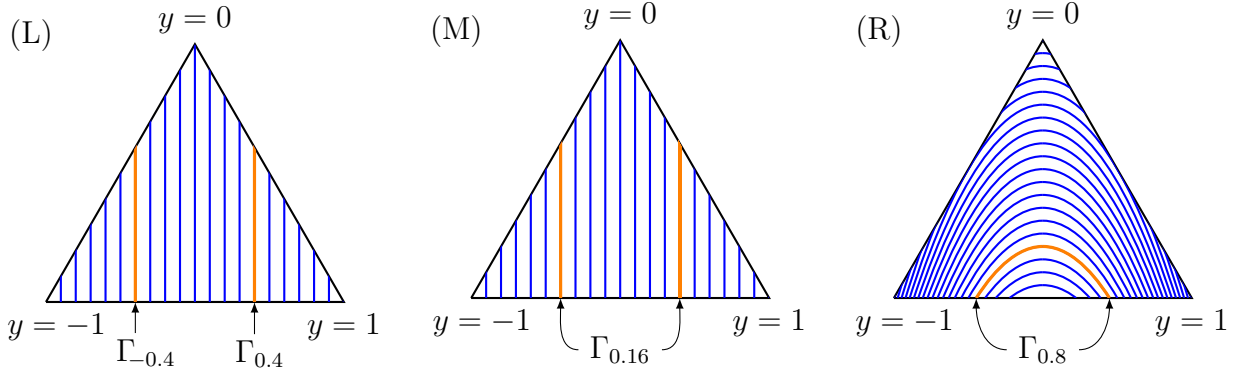

The above variance example suggests the notion of \myemph{indirect elicitation}, where we first elicit a ``intermediate'' property $\hat \Gamma$, and then use the resulting value to compute the desired property $\Gamma$.
We say a property is $k$-elicitable if it can be obtained as a function of a $k$-dimensional elicitable property.
We allow $k$ to be countably infinite, which we write $\countinf$ in lieu of the more precise countable cardinal $\aleph_0$.
The \myemph{elicitation complexity} of a property is then simply the minimum dimension $k$ needed for it to be $k$-elicitable.
Both of these definitions are only interesting when the intermediate property is restricted to some class of properties $\C$, such as those defined in \S~\ref{sec:classes-props}, as otherwise essentially all properties are 1-elicitable; see Remark~\ref{remark:everything-1-elic} in \S~\ref{sec:basic-results}.
For a discussion of other related definitions in the literature, see \S~\ref{sec:comp-other-defin}.
\begin{definition}
  \label{def:el}
  For $k\in\N\cup\{\countinf\}$, let $\EL_k(\P)$ denote the class of all elicitable properties $\Gamma:\P\to\reals^k$, and $\EL(\P) \defeq \bigcup_{k\in\N\cup\{\countinf\}} \EL_k(\P)$.
  When $\P$ is implicit we simply write $\EL$.
\end{definition}

\begin{definition}
  \label{def:elic-complex}
  Let $\C$ be a class of properties, and $k\in\N\cup\{\countinf\}$.
  A property $\Gamma:\P\to\R$ is \myemph{$k$-elicitable with respect to $\C$}
  if there exists an intermediate property $\hat \Gamma \in \C\cap\EL_k(\P)$ and map $f:\reals^k\to\R$ such that $\Gamma = f \circ \hat \Gamma$.
  The \myemph{elicitation complexity} of $\Gamma$ is $\elic_\C(\Gamma) = \min\{k: \Gamma \text{ is $k$-elicitable with respect to $\C$}\}$.
\end{definition}
\noindent
If no suitable property $\hat\Gamma$ for $\Gamma$ exists in $\C$, its elicitation complexity will be undefined.
To illustrate the definition, from the variance example above we have $\Gamma = \Var$, $\hat\Gamma: p \mapsto (\E_p[Y],\E_p[Y^2]) \in \reals^2$, and $f:(r_1,r_2) \mapsto r_2 - r_1^2$.
Hence, we conclude $\Var$ is 2-elicitable with respect to the class $\Clin$ of linear properties, i.e., expected values, which we define formally in \S~\ref{sec:classes-props}.
In particular, $\elic_\Clin(\Var) \leq 2$, meaning the elicitation complexity is at most 2.

\begin{remark}
  \label{remark:mean-squared}
  If a property is not elicitable, it can still be 1-elicitable, and thus we have not yet shown $\elic_{\C}(\Var) = 2$ for any $\C$.
  In other words, $\Gamma \notin \EL(\P)$ does not imply $\elic_\C(\Gamma) \geq 2$.
  As a simple example, consider the property $\Gamma(p) = (\E_p[Y])^2$, where $\Y = \{-1,0,1\}$.
  Clearly, the level sets of $\Gamma$ are not convex: $\Gamma((1,0,0)) = \Gamma((0,0,1)) = 1$ but $\Gamma((a,0,1-a)) < 1$ for all $0 < a < 1$; see
  Figure~\ref{fig:mean-squared-variance-lev-sets}(M).
  However, $\Gamma$ is easily indirectly elicited via $\hat\Gamma(p) = \E_p[Y] \in \reals^1$, with the simple link $f(r) = r^2$, and hence we conclude $\elic_{\C}(\Gamma) = 1$ whenever $\hat\Gamma \in \C$, such as $\C=\Clin$ being the set of linear properties, i.e., expected values.
  To show lower bounds for $\elic_\C$ we will need more tools, such as our main theorem below; see \S~\ref{sec:ex-variance} for the application to the variance.
\end{remark}

\subsection{Main Result}
\label{sec:main-result}

We now move to our main result, concerning properties that can be written as the \myemph{Bayes risk} of another loss function, the minimum possible expected loss as a function of the distribution $p$.
\begin{definition}
  \label{def:bayes-risk}
  Given loss function $L:\A\times\Y \to \reals$ for some report set $\A$, the \myemph{Bayes risk} of $L$ is defined as $\lbar(p) := \inf_{a \in \A} L(a,p)$.
\end{definition}
For example, the variance is the Bayes risk of squared loss $L(r,y) = (r-y)^2$, as we have $\lbar(p) = \min_{r\in\reals}\E_p[(r-Y)^2] = \E_p[(\E_p[Y]-Y)^2] = \Var(p)$.

Our main result gives a tight bound on the elicitation complexity of a Bayes risk.
Given a loss $L$, Theorem~\ref{thm:elic-minimum} states that its Bayes risk can be elicited jointly with the property $\Gamma$ it elicits, which implies $\elic_\C(\lbar) \leq \elic_\C(\Gamma) + 1$ whenever the pair $(\lbar,\Gamma)$ is an element of $\C$.
Theorem~\ref{thm:bayes-risk-lower-bound} gives a lower bound: for all $\C$, we have $\elic_\C(\lbar) \geq \elic_\C(\Gamma)$.
See \S~\ref{sec:elic-complex-bayes-risk} for proofs.

\begin{theorem}
  \label{thm:main-theorem}
  Let $L:\reals^k\times\Y \to \reals$ be a loss function eliciting $\Gamma:\P\to\reals^k$, $k\in\N\cup\{\infty\}$, and $\lbar$ be its Bayes risk.
  If $(\lbar,\Gamma)\in \C$ and $\elic_\C(\Gamma)=k$, then $\elic_\C(\lbar) \in \{k,k+1\}$.
  Moreover, the loss
  \begin{equation}
    \label{eq:elic-bayes-risk}
    L^*((r,a),y) = L'(a,y) + H(r) + h(r)(L(a,y) - r)
  \end{equation}
  elicits $\{\lbar,\Gamma\}$, where $h:\reals\to\reals_+$ is any positive strictly decreasing function, $H(r) = \int_0^r h(x) dx$, and $L'$ is any other loss weakly eliciting $\Gamma$.
\end{theorem}
One could easily lift the requirement that $\Gamma$ be a function, and allow $\Gamma(p)$ to be the set of minimizers of the loss \citep{frongillo2014general}; we will use this additional power in Example~\ref{sec:ex-spectral-risk}.

Meaningful applications of Theorem~\ref{thm:main-theorem} require a suitable choice of the class $\C$.
In general, the condition $(\lbar,\Gamma)\in\C$ will be true for sufficiently permissive $\C$, but the condition $\elic_\C(\Gamma)=k$ will only hold for sufficiently restrictive $\C$, and sufficiently rich $\P$.
Satisfying both conditions with the same $\C$ thus entails some understanding of the application at hand.
Before discussing several applications of Theorem~\ref{thm:main-theorem}, we first introduce the various property classes $\C$ we will focus on, and show that we can tighten our lower bound to $k+1$ for all these classes.

\subsection{Classes of Properties}
\label{sec:classes-props}

As we describe later in Remark~\ref{remark:everything-1-elic}, some restriction on $\C$ is necessary, as otherwise all properties would have complexity 1.
We focus in this paper on four natural choices of $\C$, all of interest in the machine learning literature, cf.~\citet{agarwal2015consistent}, with a discussion of other classes in \S~\ref{sec:discussion}.
Briefly, ordered from most restrictive to least restrictive, the four classes we consider are the properties which are: linear / expected values ($\Clin$), elicited by strongly convex losses ($\Cstrong$), elicited by smooth strictly convex losses ($\Cstrict$), and identifiable ($\I$).
The desired class $\C$ may depend on applications; e.g.,\ strong convexity leads to favorable optimization rates and generalization bounds for empirical risk minimization.
We now define these classes formally, beginning with the notion of identifiability.

We saw from Proposition~\ref{prop:elic-complex-level-sets-convex} that elicitable properties have convex level sets.
The class $\I$ of identifiable properties satisfy a stronger condition: not only must the level sets be convex, but they must be the intersection of a linear subspace with $\P$.
These linear subspaces are encoded by an \myemph{identification function} \citep{osband1985providing,lambert2008eliciting,steinwart2014elicitation}.
The definition we adopt corresponds to a ``strong'' identification function from Steinwart et al.~\citeyearpar{steinwart2014elicitation}.
\begin{definition}
  \label{def:prop-1}
  For $k\in\N\cup\{\countinf\}$, a $\P$-integrable function $V:\R\times\Y\to\reals^k$ is an \myemph{identification function} for $\Gamma:\P\to\reals^k$, or \myemph{identifies} $\Gamma$, if for all $r\in\Gamma(\P)$,  $p \in \Gamma_r \iff V(r,p) = 0 \in \reals^k$, where as with $L(r,p)$ above we write $V(r,p) \defeq \E_p[V(r,Y)]$.  $\Gamma$ is \myemph{identifiable} if some $V$ identifies it.
\end{definition}
\begin{definition}
  \label{def:id}
  For $k\in\N\cup\{\countinf\}$, let $\ID_k(\P)$ denote the class of all identifiable properties $\Gamma:\P\to\reals^k$, and $\ID(\P) \defeq \bigcup_{k\in\N\cup\{\countinf\}} \ID_k(\P)$.  When $\P$ is implicit we simply write $\I$.
\end{definition}
For example, $V(r,y) = y - r$ identifies the mean $\Gamma(p) = \E_p[Y]$.
More generally, the expected value $\Gamma(p) = \E_p[\phi(Y)]$ of some $\phi:\reals\to\reals^k$ has identification function $V(r,y) = r - \phi(y)$.
Similarly, when single-valued, the $\alpha$-quantile $q_\alpha(p)$, $\alpha\in(0,1)$, is identified by $V(r,y) = \ones_{Y\leq r} - \alpha$.
We may extend Definition~\ref{def:id} when $\Gamma$ is set-valued, with $\Gamma(\P)$ denoting the union of $\Gamma(p)$ for all $p\in\P$.

We now define the other three classes of properties.
Recall that a differentiable function $G:A\to\reals$ is $\mu$-strongly convex if for all $x,y\in A$ we have $\mu\|x-y\|^2 \leq (\nabla G(x) - \nabla G(y))\cdot(x-y)$.
\begin{definition}
  Let $\Clin$ denote the class of bounded linear properties, i.e.,\ those of the form $\Gamma:p\mapsto \E_p[\phi(Y)]$ for some $\P$-integrable $\phi:\Y\to\reals^k$, $k\in\N\cup\{\countinf\}$, where $\R:=\Gamma(\P)\subseteq\reals^k$ is a bounded set.
  When $k=\infty$, we use $\|\cdot\|_2$ and the Fr\'echet derivative; see \S~\ref{sec:proof-nested}.
  Let $\Cstrict$ denote the class of bounded properties $\Gamma:\P\to\R$ elicited by a loss function which is differentiable, Lipschitz-continuous, and strictly convex in the first argument.
  The class $\Cstrong \subseteq \Cstrict$ further requires the loss to be strongly convex in the first argument.
\end{definition}

As alluded to above, our four classes are nested, and each complexity therefore lower bounds the next.
We only have $\Cstrict \subseteq \I$ because we require differentiability in $\Cstrict$; removing this restriction and studying general convex losses is an important future direction (\S~\ref{sec:discussion}).
\begin{proposition}
  \label{prop:nested-convex-classes-long}
  We have $\Clin \subseteq \Cstrong \subseteq \Cstrict \subseteq \I$.
  In particular, for all properties $\Gamma$, we have $\elici(\Gamma) \leq \elic_\Cstrict(\Gamma) \leq \elic_\Cstrong(\Gamma) \leq \elic_\Clin(\Gamma)$.
\end{proposition}
The proof is straightforward (\S~\ref{sec:proof-nested}), although some care is needed in the case $k = \infty$.
We will use these relationships extensively when applying our results.
In particular, lower bounds for less restrictive classes like $\I$ are stronger, whereas upper bounds for more restrictive classes like $\Clin$ are stronger.
Moreover, as we will prove in \S~\ref{sec:elic-complex-bayes-risk}, all of the classes we consider admit a tighter lower bound of $\elic_\C(\lbar) \geq k+1$, which gives equality in light of Theorem~\ref{thm:main-theorem}.  This tighter lower bounds lower bound relies on $\P$ being sufficiently rich.  The following provides a sufficient condition.

\begin{condition}\label{cond:v-interior}
  Let $\Gamma\in\I_k(\P)$ and $r\in\Gamma(\P)$ be given.
  There exists some identification function $V:\Gamma(\P)\times\Y\to\reals^k$ identifying $\Gamma$ such that $0\in\interior\{V(r,p) : p\in\P\}$.
\end{condition}

Condition~\ref{cond:v-interior} is a weaker version of Assumption V1 of Fissler and Ziegel~\citeyearpar{fissler2016higher} as ours holds for a particular $r$ while theirs uses a universal quantifier over $r$ in the interior of $\Gamma(\P)$.
As they point out through a number of examples, such conditions are frequent in the literature on elicitation.
With this condition, we can state the tighter bound.

\begin{proposition}
  \label{prop:tighter-lower-bound}
  Let $L:\reals^k\times\Y \to \reals$ be a loss eliciting $\Gamma \in \I_k$, $k\in\N$.
  If $\Gamma$ satisfies Condition~\ref{cond:v-interior} for some $r\in\Gamma(\P)$, and $\lbar$ is non-constant on $\Gamma_r$, then $\elici(\lbar) = k+1$.
  If additionally $(\lbar,\Gamma)\in\C$ for some $\C \subseteq \I$, then $\elic_\C(\lbar) = k+1$.
\end{proposition}

\section{Examples and Applications}
\label{sec:elic-complex-exampl-appl}
\label{SEC:ELIC-COMPLEX-EXAMPL-APPL} 

\subsection{Preliminaries}

We now give several applications of our theorem.
Several upper bounds are novel, as well as all lower bounds greater than $1$.
Unless stated otherwise we will take $\Y=\reals$.  In each setting, we also make several standard regularity assumptions which we suppress for ease of exposition; for example, for the variance and variantile we assume finite first and second moments.
All applications also require $\P$ to be ``sufficiently rich'' in some sense, typically to establish $\elic_\C(\Gamma)=k$, which is often a light restriction.
For example, in many cases our results hold for any $\P$ containing the set $\gaussians$ of all finite mixtures of Gaussian distributions.
We will defer these richness conditions to the following section, in particular Conditions~\ref{cond:quantiles} and~\ref{cond:linear}, and instead refer to the results that use these definitions to establish basic complexity bounds, such as Lemmas~\ref{lem:complex-whole-dist} and~\ref{lem:elic-complex-quantiles}.
For omitted proofs and other details, see Appendix~\ref{sec:omitted-appl}.

\subsection{Variance}
\label{sec:ex-variance}

Following Definition~\ref{def:elic-complex}, we noted that the variance is a function of the first and second moment, which are both linear properties, giving us $\elic_\Clin(\Var)\leq 2$.
As a warm up, let us see how to apply our main theorem to recover this statement together with a matching lower bound.
As we saw above, we can view the variance as the Bayes risk of squared loss $L(r,y) = (r-y)^2$, which of course elicits the mean.
As the mean is identifiable, and the variance is not simply a function of the mean, Proposition~\ref{prop:tighter-lower-bound} gives $\elic_\I(\Var)=2$.
Furthermore, we can directly establish $\elic_\Clin(\Var) \leq 2$.
Letting $\hat\Gamma(p) = \{\E_p[Y],\E_p[Y^2]\}$ be the first and second moment, we have $\hat\Gamma\in\Clin$ and $\Var = f \circ \hat\Gamma$ for $f : (r_1,r_2) \mapsto r_2 - r_1^2$.
Proposition~\ref{prop:nested-convex-classes-long} then gives $\elic_\C(\Var)=2$ for any class $\C$ between $\Clin$ and $\I$, including all $\C\in\{\Clin,\Cstrong,\Cstrict,\I\}$.

\begin{corollary}
  \label{cor:variance}
  Let $\P$ contain $\gaussians$, or any set of distributions such that (i) Condition~\ref{cond:v-interior} holds for the mean $\Gamma:p\mapsto \E_p[Y]$ and some $r\in\reals$, and (ii) there are two distributions with mean $r$ but different variances.
  Then $\elic_\C(\Var)=2$ for all $\Clin \subseteq \C \subseteq \I$.
\end{corollary}

With the variance we can observe that our Theorem~\ref{thm:main-theorem} does not always give a full characterization of loss functions eliciting $(\lbar,\Gamma)$.
For $(\Var,\E[Y])$, while Theorem~\ref{thm:main-theorem} generates losses such as $L^*((r,a),y) = e^{-r}((a-y)^2-r)-e^{-r}$, there are losses which cannot be represented by the form~\eqref{eq:elic-bayes-risk}.
Perhaps the most natural example is the following,
\begin{equation}
  L^*((r,a),y) = (a-y)^2 + (r+a^2-y^2)^2~,\label{eq:brier-brier}
\end{equation}
which is given by applying the invertible link function $(m_1,m_2)\mapsto(m_1,m_2-m_1^2)$ to the loss $\hat L((m_1,m_2),y) = (m_1-y)^2 + (m_2 - y^2)^2$, which elicits $\hat\Gamma$ above.
Finally, one may be tempted to nest squared loss $L^*((r,a),y) = ((a-y)^2-r)^2$, which is similar to eq.~\eqref{eq:brier-brier}, but even after removing the $(a-y)^4$ term this loss fails because the coefficient of $(a-y)^2$ is negative.

\subsection{Entropy and Norms}
\label{sec:ex-cvx-ftns-of-means}

To demonstrate the ability of our framework to show that some properties of interest are inherently hard to elicit, consider eliciting the entropy or a norm of a distribution.
Both are used as measures of information or non-uniformity, and in their relative forms as measures of distance.
We show that these have maximum elicitation complexity, meaning there is no better way to elicit them than to first elicit the full distribution.
This result is a consequence of a more general characterization of the elicitation complexity of properties which can be written as the Bayes risk of a loss eliciting a linear property, i.e., an expectation.

The notion of entropy, as measuring disorder, randomness, information, etc., appears throughout the sciences.
As a function of a distribution over $\Y=\reals$ admitting a continuous density $p$, some standard examples include 
Shannon entropy $H(p) = -\int_\Y p(y) \log p(y) dy$, Tsallis/Havrda--Charvát entropy $H_{HC}(p) = \tfrac 1 {1-\alpha} (1-\int_\Y p(y)^\alpha dy)$ for $\alpha \neq 1$, and Rényi entropy $H_R(p) = \tfrac 1 {1-\alpha} \log(\int_\Y p(y)^\alpha dy)$ for $\alpha\geq 0, \alpha\neq 1$.
Each concave entropy function also gives rise to a corresponding entropy relative to some other distribution $q$, the most common example being Kullback--Leibler divergence $D_{\text{KL}}(p\parallel q)=\int _{-\infty }^{\infty }p(x)\log \frac {p(x)}{q(x)}\,dx$.
Similarly, norms of distributions are ubiquitous, such as the standard $\|p\|_\beta = (\int_\Y p(y)^\beta dy)^{1/\beta}$ for $\beta > 0$, and are used in their relative forms as measuring distance from some other distribution $q$.
When $|\Y|<\infty$, we simply replace integrals with sums, so that $H(p) = -\sum_{y\in\Y} p(y) \log p(y)$ and $\|p\|_\beta = (\sum_{y\in\Y} p(y)^\beta)^{1/\beta}$.

Essentially all of these entropies and norms have maximal elicitation complexity, being as hard to elicit as the distribution itself, i.e., the property $\idprop:p\mapsto p$.
From standard results in proper scoring rules \citep{gneiting2007strictly}, any strictly concave function $G:\P\to\reals$ is the Bayes risk $\lbar(p) = \E_p L_G(p,Y)$ of some strictly proper loss $L_G$ which elicits $\idprop$.
For example, Shannon entropy is the Bayes risk of log loss, $L(p,y) = -\log p(y)$, which elicits $\idprop$.
Moreover, under suitable richness conditions, we have $\elic_\C(\idprop) = \infty$ for all $\Clin\subseteq\C\subseteq\I$ by Lemma~\ref{lem:complex-whole-dist}, or $\elic_\C(\idprop) = |\Y|-1$ when $\Y$ is a finite set.
Finally, since clearly $G = G \circ \idprop$, the result developed for our main lower bound, from Theorem~\ref{thm:bayes-risk-lower-bound}, gives
$\elic_\C(\lbar) = \elic_\C(\idprop)$ for all $\C$.

\begin{corollary}
  \label{cor:entropy-norm}
  Let $\C$ satisfy $\Clin\subseteq\C\subseteq\I$, and let $G:\P\to\reals$ be strictly convex.
  Then $\elic_C(G) = \elic_\C(\idprop)$.
  If $|\Y|<\infty$ and $\P$ is the probability simplex, $\elic_\C(G) = |\Y|-1$.
  If $\Y = \reals$ and $\P$ is a convex family of Lebesgue densities satisfying appropriate richness conditions, as in Lemma~\ref{lem:complex-whole-dist}, then
  $\elic_\C(G) = \countinf$.
\end{corollary}

Corollary~\ref{cor:entropy-norm} applies to the entropies and norms above when we choose parameters making them strictly concave or convex, namely $\alpha < 1$ for $H_{HC}$, $\alpha\neq 1$ for $H_R$, and $\beta>1$ \citep{rao1984convexity}.
The result generalizes to any strictly convex function of expected values, as outlined in \S~\ref{sec:conv-funct-means}.
See also \S~\ref{sec:comp-other-defin} for a related discussion of multi-observation losses~\citep{casalaina-martin2017multi}.

\subsection{Expected Shortfall, Spectral Risk Measures, and Range Value at Risk}
\label{sec:ex-es}
\label{sec:ex-spectral-risk}

One important application of our results on the elicitation complexity of the Bayes risk is the elicitability of various financial risk measures.  One of the most popular financial risk measures is \myemph{expected shortfall} $\ES_\alpha:\P\to\reals$, also called \myemph{conditional value at risk (CVaR)} or \myemph{average value at risk (AVaR)}, which we define as follows; cf.~\citet[eq.(18)]{follmer2015axiomatic}, \citet[eq.(3.21)]{rockafellar2013fundamental}.
\begin{align}
  \ES_\alpha(p)
  &= \inf_{z\in\reals}\left\{ \E_p\left[ \tfrac 1 \alpha(z-Y)\ones_{z\geq Y}- z\right]\right\}
  = \inf_{z\in\reals}\left\{ \E_p\left[ \tfrac 1 \alpha(z-Y)(\ones_{z\geq Y}-\alpha) -Y\right]\right\}~.
  \label{eq:elic-complex-es-3}
\end{align}
We will assume $\Y=\reals_+$, the nonnegative reals, and restrict $\alpha\in(0,1)$; see below for $\alpha=1$.
Despite the importance of elicitability to financial regulation \citep{emmer2015what,fissler2016expected}, $\ES_\alpha$ is not elicitable \citep{gneiting2011making}.
It was recently shown by Fissler and Ziegel~\citeyearpar{fissler2016higher}, however, that $\elici(\ES_\alpha) \leq 2$.  They also consider the broader class of {\em spectral risk measures}, which can be represented as
$\rho_\mu(p) = \int_{(0,1)} \ES_\alpha(p) d\mu(\alpha)$,
where $\mu$ is a probability measure on $(0,1)$; cf. \citet[eq.~(36)]{follmer2015axiomatic}.
In the case of finite support, $\mu = \sum_{i=1}^k \beta_i \delta_{\alpha_i}$, for distinct point distributions $\delta_{\alpha_i}$, $\beta_i > 0$, we can rewrite $\rho_\mu$ using the above as:
\begin{align}
  \rho_\mu(p)
  &= \sum_{i=1}^k \beta_i \ES_{\alpha_i}(p)
    = \inf_{z\in\reals^k} \left\{ \E_p\left[ \sum_{i=1}^k \frac{\beta_i}{\alpha_i} (z_i-Y)(\ones_{z_i\geq Y}-\alpha_i) -Y\right]\right\}~.
  \label{eq:elic-complex-spec-1}
\end{align}
Fissler and Ziegel then conclude $\elici(\rho_\mu) \leq k+1$.

We show how to recover these results as well as matching lower bounds.
Let $\Pquant$ be the set of probability measures over $\reals$ with single-valued quantiles in the range $(0,1)$, i.e., supported on an interval and whose CDFs are strictly increasing on that interval.
It is well-known that the infimum in eq.~\eqref{eq:elic-complex-spec-1} is attained by the $k$ distinct quantiles $q_{\alpha_1}(p),\ldots,q_{\alpha_k}(p)$.
Thus, we may express $\rho_\mu$ as a Bayes risk; in particular, $\rho_\mu(p) = \lbar(p)$ for the the loss $L:\reals^k\times\reals_+$ given by
\begin{align}
  L(z,y) &= \sum_{i=1}^k \frac{\beta_i}{\alpha_i} (z_i-y)(\ones_{z_i\geq Y}-\alpha_i) -y~,
  \label{eq:loss-quantiles}
\end{align}
which elicits $\Gamma(p) = \{q_{\alpha_1}(p),\ldots,q_{\alpha_k}(p)\}$.
As $\Gamma$ is identifiable by assumption on $\P$, and we have $\elici(\Gamma) = k$ when $\P$ is sufficiently rich, as in Lemma~\ref{lem:elic-complex-quantiles},
Proposition~\ref{prop:tighter-lower-bound} gives us $\elici(\rho_\mu)= k+1$.
In particular, the property $\{\rho_\mu,q_{\alpha_1},\ldots,q_{\alpha_k}\}$ is elicitable.
Moreover, in \S~\ref{sec:elic-complex-loss-expect-shortf} we show that the family of losses from Theorem~\ref{thm:main-theorem} coincide with the characterization of Fissler and Ziegel~\citeyearpar{fissler2016higher}.

\begin{corollary}
  \label{cor:spectral-risks}
  Let $\P\subseteq\Pquant$ be sufficiently rich, as in Lemma~\ref{lem:elic-complex-quantiles}, and contain all mixtures of Pareto distributions, or any set of distributions where there are at least two possible $\rho_\mu$ values for a given vector of quantiles $q_{\alpha_1}(p),\ldots,q_{\alpha_k}(p)$.
  Then $\elici(\rho_\mu) = k+1$.
\end{corollary}

Unlike the previous examples, here we only have a tight result when $\C=\I$.
While we have $\elic_\C(\rho_\mu) \geq k+1$ for any $\C \subseteq \I$, including the classes $\Cstrict$, $\Cstrong$, and $\Clin$, the upper bound $\elic_\C(\rho_\mu) \leq k+1$ only holds for $\C = \I$ among these four classes.
The reason for this difference is simply the fact that the losses from Theorem~\ref{thm:main-theorem} are not strictly convex, and thus the condition $(\rho_\mu,\Gamma)\in\C$ is not established for $\C \subseteq \Cstrict$.

\begin{remark}
\label{remark:spectral-risk-bound}
When $\alpha=1$, we have $\ES_1(p) = \E_p[-Y]$, and thus $\elici(\rho_\mu)=1$ for $\mu(\{1\})=1$.
Moreover, when $\mu(\{1\})\in(0,1)$, we simply replace the loss in eq.~\eqref{eq:loss-quantiles} by $L(z,y) = \sum_{i=1}^{k-1} \frac{\beta_i}{\alpha_i} (z_i-y)(\ones_{z_i\geq Y}-\alpha_i) -(1+\beta_k) y$, yielding a bound of $\elici(\rho_\mu)=k$ when $\mu(\{1\}) > 0$, as opposed to $\elici(\rho_\mu)=k+1$ when $\mu(\{1\}) = 0$; cf.~\citet[Corollary 5.4(ii)]{fissler2016higher}.
\end{remark}

Finally, concurrent with our work, Fissler and Ziegel~\citeyearpar{fissler2019elicitability} give a result for \myemph{Range Value at Risk (RVaR)}, which motivates a certain generalization of our upper bound, Theorem~\ref{thm:elic-minimum}.
Thought of as a compromise between VaR and ES, RVaR is defined as follows for $0 <\alpha <\beta <1$,
\begin{equation}
  \label{eq:rvar-es}
  \mathrm{RVaR}_{\alpha,\beta}(p) :=
  \frac{1}{\beta-\alpha} \int_\alpha^\beta \mathrm{VaR}_\lambda(p)d\lambda
  =
  \frac{\beta \ES_\beta(p) - \alpha \ES_\alpha(p)}{\beta-\alpha}
  ~,
\end{equation}
where the second equality holds whenever the right-hand side is defined \citep{fissler2019elicitability}.
While $\ES$ is a Bayes risk, as noted above, the form~\eqref{eq:rvar-es} is a difference of Bayes risks and thus Theorem~\ref{thm:main-theorem} does not apply.
The discussion above on the complexity of $\ES$, together with Lemma~\ref{lem:complex-sub-add} below on the subadditivity of $\elic_\C$, still gives $\elici(\mathrm{RVaR}_{\alpha,\beta}) \leq \elici(\ES_\alpha) + \elici(\ES_\beta) = 4$,
which the authors note has been observed and used in practice; specifically, the quadruple $(\mathrm{VaR}_\alpha,\mathrm{VaR}_\beta,\ES_\alpha,\ES_\beta)$ is elicitable.
The authors improve on this complexity by showing that $(\mathrm{VaR}_\alpha,\mathrm{VaR}_\beta,\mathrm{RVaR}_{\alpha,\beta})$ is elicitable, so that $\elici(\mathrm{RVaR}_{\alpha,\beta}) \leq 3$.
See Wang and Wei~\citeyearpar{wang2018risk} for a perspective on this result in the broader context of signed Choquet integrals.

This interesting case gives rise to a generalization of the upper bound from Theorem~\ref{thm:main-theorem}: linear combinations of Bayes risks are elicitable along with the corresponding properties.
The proof (\S~\ref{sec:proof-lin-comb-bayes-risks}) adapts Theorem~\ref{thm:main-theorem} with additional terms to account for possibly negative coefficients.

\begin{theorem}\label{thm:lin-comb-bayes-risk}
For each $i\in\{1,\ldots,m\}$ let $L_i:\reals^{k_i}\times\Y \to \reals$ be a loss eliciting $\Gamma_i:\P\to\reals^{k_i}$, with Bayes risk $\lbar_i$.
Let $\gamma(p) = \sum_{i=1}^m \alpha_i \lbar_i(p)$ for $\alpha_i \in \reals\setminus\{0\}$.
Then $\{\gamma,\Gamma_1,\ldots,\Gamma_m\}$ is elicitable.
  In particular, if $\{\gamma,\Gamma_1,\ldots,\Gamma_m\}\in\C$, $\elic_\C(\gamma) \leq \sum_{i=1}^m k_i+1$.
\end{theorem}

Returning to RVaR, we have $\Gamma_1 = \mathrm{VaR}_\alpha$, $\lbar_1 = \ES_\alpha$, $\Gamma_2 = \mathrm{VaR}_\beta$, $\lbar_2 = \ES_\beta$, and take $\alpha_1 = \alpha/(\alpha-\beta) < 0$ and $\alpha_2 = \beta/(\beta-\alpha)>0$.
Theorem~\ref{thm:lin-comb-bayes-risk} then recovers the elicitability of $(\mathrm{VaR}_\alpha,\mathrm{VaR}_\beta,\mathrm{RVaR}_{\alpha,\beta})$ and $\elici(\mathrm{RVaR}_{\alpha,\beta}) \leq 3$.
Moreover, the scope of loss functions (\S~\ref{sec:proof-lin-comb-bayes-risks},\S~\ref{sec:elic-complex-loss-expect-shortf}) matches those found by Fissler and Ziegel~\citeyearpar{fissler2019elicitability}.
Unlike our other examples, however, it is unclear how to prove lower bounds on the complexity of RVaR or other linear combinations of Bayes risks; this is an interesting direction for future work.

\subsection{A New Risk Measure: The Variantile}
\label{sec:ex-variantile}

The $\tau$-expectile, denoted $\mu_\tau$, is a type of generalized quantile introduced by Newey and Powell~\citeyearpar{newey1987asymmetric}, is defined as the solution $x$ to the equation $\E_p\left[|\ones_{x\geq Y}-\tau|(x-Y)\right]=0$, where $\tau \in (0,1)$, which also shows $\mu_\tau\in\ID_1$.
Here we propose the $\tau$-\myemph{variantile}, an asymmetric variance-like measure analogous to the $\tau$-expectile: just as the mean is the solution $x=\mu$ to the equation $\E_p[x-Y]=0$, and the variance is $\Var(p) = \E_p[(\mu-Y)^2]$, we define the $\tau$-variantile $\Var_\tau$ by $\Var_\tau(p) = 2 \E_p\left[|\ones_{\mu_\tau\geq Y}-\tau|(\mu_\tau-Y)^2\right]$.
As the expectile can be thought of as a compromise between the mean and a quantile, the variantile can be thought of a compromise between the variance, recovered by $\tau = 0.5$, and the variance of a ``superquantile''; see \S~\ref{sec:ex-es}.
Therefore, variantiles may have applications as a new tractable measure of risk.  (During the final preparation of this paper for publication, we learned that this same concept was previously proposed in unpublished work by Wei Hu and Zhenlong Zheng as the ``variancile''.)

It is well-known that $\mu_\tau$ can be expressed as the minimizer of a \myemph{asymmetric least-squares} problem: the loss $L(x,y) = |\ones_{x\geq y}-\tau|(x-y)^2$ elicits $\mu_\tau$ \citep{newey1987asymmetric,gneiting2011making}.  Hence, as the variance is in fact a Bayes risk for the mean, so is the $\tau$-variantile for the $\tau$-expectile:
\begin{align*}
  \mu_\tau(p) = \argmin_{x\in\reals} \;2 \E_p\left[|\ones_{x\geq
      Y}-\tau|(x-Y)^2\right]
  \\
  \Var_\tau(p) = \min_{x\in\reals} \;2 \E_p\left[|\ones_{x\geq
      Y}-\tau|(x-Y)^2\right]~.
\end{align*}
We now see the pair $\{\mu_\tau,\Var_\tau\}$ is elicitable by Theorem~\ref{thm:main-theorem}, and as $\mu_\tau \in \I$ we obtain a tight complexity bound with respect to $\I$ from Proposition~\ref{prop:tighter-lower-bound}.
Moreover, if $\Y$ is bounded, we have $\{\mu_\tau,\Var_\tau\} \in \Cstrong$ from Proposition~\ref{prop:str-cvx-elic} below, which gives conditions under which the loss $L^*$ in eq.~\eqref{eq:elic-bayes-risk} can be taken to be strongly convex;
in this case, we have tight bounds for $\Cstrong$ and $\Cstrict$ as well.
See \S~\ref{sec:omitted-variantiles} for the full proof.

\begin{corollary}
  \label{cor:expectile}
  Let $\P$ contain $\gaussians$, or any set of distributions such that (i) Condition~\ref{cond:v-interior} holds for the $\tau$-expectile and some $r\in\reals$, and (ii) there are at least two distributions with $\tau$-expectile $r$ but different $\tau$-variantiles.
  Then $\elici(\Var_\tau)=2$.
  If additionally $\Y\subseteq\reals$ is bounded, thereby excluding $\gaussians$, then $\elic_\C(\Var_\tau) = 2$ for all $\C$ satisfying $\Cstrong \subseteq \C \subseteq \I$.
\end{corollary}

More generally, Herrmann et al.~\citeyearpar{herrmann2018multivariate} introduce a multivariate expectile.
Observing that univariate asymmetric least-squares can be written $L(x,y) = \tfrac 1 2 |y-x|(|y-x|+(2\tau-1)(y-x))$, they generalize this loss to higher dimensions by replacing $|\cdot|$ with $\|\cdot\|_2$ and letting $2\tau-1$ now be an arbitrary vector in the open unit ball, just as $-1 < 2\tau-1 < 1$.
The minimizer of this loss is the multivariate expectile, $\mu^{(k)}_\tau(p)$, where $k$ is the dimension of the vector space.
We can analogously define our multivariate variantile; the pair are given as follows,
\begin{align}
  \mu^{(k)}_\tau(p) = \argmin_{x\in\reals^k} \;2\E_p\left[\|Y-x\|_2(\|Y-x\|_2+\inprod{\tau,Y-x})\right] \label{eq:mult-expectile}
  \\
  \Var^{(k)}_\tau(p) = \min_{x\in\reals^k} \;2\E_p\left[\|Y-x\|_2(\|Y-x\|_2+\inprod{\tau,Y-x})\right]~, \label{eq:mult-variantile}
\end{align}
where now $Y\in\reals^k$, and $\tau\in\reals^k$ is a vector in the open unit ball, i.e.,\ $\|\tau\|_2<1$.
We again obtain a tight complexity bound, which as in the univariate case holds with respect to $\I$ unconditionally, and with respect to $\Cstrict$ and $\Cstrong$ when $\Y$ is bounded.

\begin{corollary}
  \label{cor:multivariate-expectile}
  Let $\P$ contain $\gaussians$, or any set of distributions such that (i) Condition~\ref{cond:v-interior} holds for $\mu^{(k)}_\tau$ and some $r\in\reals^k$, and (ii) there are at least two distributions $p,p'\in\P$ with $\mu^{(k)}_\tau(p)=\mu^{(k)}_\tau(p')=r$ but $\Var^{(k)}_\tau(p)\neq\Var^{(k)}_\tau(p')$.
  Then $\elici(\Var^{(k)}_\tau) = k+1$.
  If additionally $\Y\subseteq\reals$ is bounded, thereby excluding $\gaussians$, then $\elic_\C(\Var^{(k)}_\tau) = k+1$ for all $\C$ with $\Cstrong \subseteq \C \subseteq \I$.
\end{corollary}

\subsection{Other Risk Measures}
\label{sec:ex-rockafellar}
\label{sec:ex-coherent}

\newcommand{\cR}{\mathcal{R}}
\newcommand{\cD}{\mathcal{D}}
\newcommand{\cE}{\mathcal{E}}
\newcommand{\cS}{\mathcal{S}}

Several other risk measures have appeared in the literature in finance and engineering.
For example, consider the broad class risk measures arising from the ``risk quadrangles'' of Rockafellar and Uryasev~\citeyearpar{rockafellar2013fundamental}, which are given by the following relationships between a risk $\cR$, deviation $\cD$, error $\cE$, and a statistic $\cS$, all functions from random variables to the reals:
\begin{align*}
  \cR(X) = \min_{c\in\reals}\{c + \cE(X-c)\}~,
           \quad
  \cD(X) = \min_{c\in\reals}\{\cE(X-c)\}~,
           \quad
  \cS(X) = \argmin_{c\in\reals}\{\cE(X-c)\}~.
\end{align*}
Fixing a particular form for $\cE$ then fixes the other three.
Our results apply readily to the \myemph{expectation quadrangle} case, where $\cE(X) = \E[e(X)]$ for some $e:\reals\to\reals$.
Here we consider $\cR$ and $\cD$ as functions of the distribution of $X$, which is possible here as they are both law-invariant when $\cE$ is of expectation type; see \S~\ref{sec:elic-complex-setting}.
Under appropriate conditions, Proposition~\ref{prop:tighter-lower-bound} then implies $\elici(\cR) = \elici(\cD) = 2$ provided $\cS$ is non-constant and identifiable.
This statement covers several of their examples, such as the truncated mean, log-exp, and rate-based.
Beyond the expectation case, the authors show a Mixing Theorem, where they consider
\begin{align*}
  \cD(X) = \min_{c\in\reals} \: \min_{b_1,..,b_k \in \reals} 
           \left\{ \sum_{i=1}^k \lambda_i \cE_i(X-c-b_i) \,\big|\, \sum_i\lambda_iB_i = 0\right\}
         = \min_{b_1',..,b_k' \in \reals}
  \left\{ \sum_{i=1}^k \lambda_i \cE_i(X-b_i') \right\}~.
\end{align*}
Once again, if the $\cE_i$ are all of expectation type and the $\cS_i$ identifiable, Theorem~\ref{thm:main-theorem} gives $\elici(\cD)=\elici(\cR)\leq k+1$, with a matching lower bound from Proposition~\ref{prop:tighter-lower-bound}, under appropriate assumptions, provided the $\cS_i$ are all independent (Definition~\ref{def:prop-indep}).
Finally, the Reverting Theorem for a pair $\cE_1,\cE_2$ can be seen as a special case of the above where one replaces $\cE_2(X)$ by $\cE_2(-X)$.
Consequently, our results give tight complexity bounds for several other examples, including ``superquantiles'' or spectral risk measures, the quantile-radius quadrangle, and optimized certainty equivalents of Ben-Tal and Teboulle~\citeyearpar{ben-tal2007old}.

Our results explain the existence of regression procedures for some of these risk/deviation measures.
For example, Rockafellar et al.~\citeyearpar{rockafellar2014superquantile} introduce \myemph{superquantile regression} to fit models to spectral risk measures.
Superexpectations are another example \citep{rockafellar2013superquantiles}.
In light of Theorem~\ref{thm:main-theorem}, one could interpret superquantile regression as simply performing regression on the $k$ different quantiles in tandem with their joint Bayes risk.
In fact, our results show that any risk/deviation generated by mixing several expectation quadrangles will have a similar procedure, in which the $b_i'$ variables are simply computed along side the measure of interest.
Even more broadly, such regression procedures exist for \myemph{any} Bayes risk.

Finally, we briefly consider \myemph{coherent} risk measures, a class containing spectral risk measures and several other examples above.
Among other properties, coherent risk measures satisfy positive homogeneity, in the sense that $\rho(\alpha X) = \alpha\rho(X)$ where $\alpha \geq 0$.
Coherent risk measures can be characterized by their well-known dual representation,
$\rho(X) = \sup_{Q\in\Qc} \E[Q X]$,
where $\Qc$ is a convex set of random variables called the \myemph{risk envelope} \citep{follmer2004stochastic,ang2018dual}.
Despite the similarity of this representation to eq.~\eqref{eq:prop-minimum}, Theorem~\ref{thm:main-theorem} typically does not apply directly, as often the envelope $\Qc$ is an infinite-dimensional set, yielding trivial upper bounds.
For example, expected shortfall at level $\alpha$ is usually given with $\Qc = \{ Q : 0 \leq Q \leq 1/\alpha \}$ \citep{delbaen2002coherent,ang2018dual}.
That said, if the potential optimizers within $\Qc$ can be parameterized by a finite-dimensional parameter, as we saw for expected shortfall in eq.~\eqref{eq:elic-complex-es-3}, and sufficient continuity holds with respect to that parameter, the theorem would apply.

\subsection{Empirical Risk Minimization}
\label{sec:erm}

Recall that in many statistical learning settings, one wishes to learn a model or \myemph{hypothesis} $h:\X\to\R$ from a class $\H$ to predict a value in $\R$ as a function of a feature vector $x\in\X$.
For example, linear classification has $\X=\reals^d$ and $\Y=\R=\{+1,-1\}$, with hypothesis class $\H_{\text{lin}} = \{h_\theta: x \mapsto \sgn(x\cdot\theta + b) \mid \theta\in\reals^m, b\in\reals\}$.
The prediction error of a hypothesis $h$ is judged by some given loss $\ell : \R \times \Y \to \reals$, such as the 0-1 loss $\ell(r,y) = \ones\{r\neq y\}$ in classification.
Letting $P$ be the underlying distribution over $\X\times\Y$, one therefore seeks a hypothesis $h\in\H$ which minimizes the expected loss $\E_P \ell(h(X),Y)$.

Many algorithms to solve this learning problem fall under the broad umbrella of (regularized) empirical risk minimization, where given a finite data set $D = \{(x_i,y_i)\}_{i=1}^n$, one chooses
\begin{equation}
  \label{eq:ERM}
  h^* \in \argmin_{h\in\H} \sum_{(x_i,y_i)\in D} \ell(h(x_i),y_i) + g(h)~,
\end{equation}
where $g:\H\to\reals$ is a regularizer.
The optimization problem in eq.~\eqref{eq:ERM} can be intractable, however, especially when $\R$ is a finite set, as in classification, ranking, and related problems \citep{arora1993hardness}.
A common approach therefore is instead to find a \myemph{surrogate} loss $L:\reals^k\times\Y\to\reals$ which is easier to optimize, and to choose the hypothesis which minimizes the empirical $L$ loss, followed by a \myemph{link function} $f:\reals^k\to\R$ \citep{bartlett2006convexity}.
For example, support vector machines (SVMs), boosting, and logistic regression can all be seen as optimizing convex surrogate losses over $\reals$, followed by the link $f=\sgn:\reals\to\{+1,-1\}$.
See below for more on SVMs.

This surrogate procedure raises the following question: when does optimizing the surrogate loss $L$ and applying some link $f$ achieve the optimal $\ell$ loss, or in other words, when is $L$ \myemph{calibrated}?
There are at least three interesting ways to make this question precise.
The weakest is that exactly minimizing $L$ and then applying $f$ exactly minimizes $\ell$, for all distributions over the outcomes $\Y$.
Stronger, we can require asymptotic calibration, that any sequence that converges to the minimum of $L$, when composed with $f$, also converges to the minimum of $\ell$.
Stronger still, we can seek rates at which this convergence occurs.

All of these formulations have connections to elicitation complexity.
Let $\Gamma^L$ and $\Gamma^\ell$ be the possibly set-valued properties elicited by $L$ and $\ell$, respectively.
The weakest relationship above, that exactly minimizing $L$ and applying $f$ exactly minimizes $\ell$, holds if and only if $\Gamma^L$ refines $\Gamma^\ell$, in the sense that for all $u\in\reals^k$ there exists an $r\in\R$ such that $\Gamma^L_u \subseteq \Gamma^\ell_r$; see Definition~\ref{def:refine}.
For example, if one seeks a smooth strictly convex loss $L:\reals^k\times\Y\to\reals$ which is calibrated in this weak sense with respect to $\ell$, then the minimum possible value of the dimension $k$ is precisely the elicitation complexity $\elic_\Cstrict(\Gamma^\ell)$.

For asymptotic calibration, there is an additional requirement that $f$ and $\ell$ satisfy some type of continuity.
Intuitively, if $\ell \circ f \circ \Gamma^L$ is not continuous, one may be able to minimize $L$ arbitrarily well but still be far from minimizing $\ell$.
As a simple example for $\R=\Y=\reals$ and $k=1$, consider $\ell(r,y) = \ones\{r\neq y\}$ and $L(u,y) = (u-y)^2$.
Agarwal and Agarwal~\citeyearpar{agarwal2015consistent} give such a condition for classification-like problems.
The general version corresponds to the existence of a strictly positive calibration function \citep{steinwart2008support}.
Rates typically rely on a stronger uniform continuity property, e.g., Theorem 3.22 of Steinwart and Christmann~\citeyearpar{steinwart2008support}.

As a concrete example, consider the hinge loss $L(u,y) = \max\{0,1-uy\}$ where $\Y=\{+1,-1\}$ and $u\in\reals$.
As discussed above, SVMs use hinge loss as a convex surrogate for 0-1 loss $\ell(r,y) = \ones\{r\neq y\}$, where the surrogate minimization is followed by the link $f(u) = \sgn(u)$.
Let us verify that the various relationships hold between the minimizers of these losses.
After clipping $u$ to the range $[-1,1]$, as all other values of $u$ are weakly dominated, we can describe the property $\Gamma^L$ elicited by the hinge loss, and its level sets $\Gamma^L_u$, as follows:
\begin{align}
  \label{eq:hinge}
  \Gamma^L(p) &=
  \begin{cases}
    -1 & 0 \leq p(+1) < 1/2\\
    [-1,1] & p(+1) = 1/2\\
    1 & 1/2 < p(+1) \leq 1
  \end{cases}~,
  \quad\;\;
  \Gamma^L_u =
  \begin{cases}
    \{p : p(-1) \geq 1/2\} & u = -1\\
    \{p : p(+1) \geq 1/2\} & u = 1\\
    \{(1/2,1/2)\} & u \in (-1,1)
  \end{cases}~.
\end{align}
By inspection, we have $\Gamma^L_u \subseteq \Gamma^\ell_r$ for $r =\sgn(u)$, implying the link function $f = \sgn$.
Moreover, \citet[Theorem 3.34, 3.36]{steinwart2008support} show that hinge loss achieves asymptotic, and indeed uniform, calibration.

These observations show that, fundamentally, the surrogates $L$ for $\ell$ which lead to consistent learning algorithms depend on $\Gamma^L$, $f$, and $\Gamma^\ell$, rather than $L$ directly.
Implicit in this claim, however, is the assumption that the learning algorithm is considering an unrestricted class $\H$ of models.
If the model class is restricted, such as for $\H_{\text{lin}}$ above, we are not guaranteed that the optimal map $h^*: x \mapsto \Gamma^L(p_x)$, where $p_x=\Pr[Y=y\mid X=x]$ is the true distribution over $y$ values, will be in $\H$.
In this case, consistency is much harder to establish, and in particular, different choices of surrogates which elicit $\Gamma^L$ will affect the final $\ell$-risk achieved.
Therefore, tools which provide a variety of loss functions can also be important.

In other learning settings, the natural problem is not necessarily to minimize a particular loss $\ell$, but instead to estimate a given statistic.
For example, in regression, for a given $x\in\reals^d$ there will typically be a distribution over $y\in\reals$ values in the population, and we are given some summary statistic of interest, such as the mean.
In these settings, it is natural to specify the problem directly in terms of the desired property $\Gamma$ and seek an elicitable $\Gamma^L$ and link $f$ such that $\Gamma = f \circ \Gamma^L$.
As long as $\Gamma$ satisfies suitable continuity properties, learning guarantees similar to consistency can be provided.

In summary, therefore, upper bounds on $\elic_\C(\Gamma)$ often give statistically consistent surrogate losses for a given property of interest $\Gamma$, where $\Gamma = \Gamma^\ell$ if a loss $\ell$ is given instead.
Moreover, an upper bound $\elic_\C(\Gamma) \leq k$ implies that the intermediate property $\Gamma^L$ is a function to $\reals^k$, meaning the dimension of the range of the underlying hypothesis can be taken to be at most $k$.
Note that $k$ is not the number of parameters, which for $\H_{\text{lin}}$ was $d+1$.
Similarly, lower bounds $\elic_\C(\Gamma) \geq k$ show that for any such surrogate loss and link to exist, with respect to the class $\C$, then the dimension of the hypothesis range must be at least $k$.

\section{Basic Complexity Results}
\label{sec:basic-results}
\label{sec:ident-elic}

\subsection{Initial Observations}

We begin with an important point: without any restriction on the class of properties $\C$, Definition~\ref{def:elic-complex} becomes trivial and all properties become 1-elicitable.
This observation does not subsume Remark~\ref{remark:mean-squared} about the case $\Gamma(p) = (\E_p[Y])^2$, as there we can show $\elic_\Clin(\Gamma)=1$.

\begin{remark}
  \label{remark:everything-1-elic}
  The set-theoretic cardinalities of $\reals$ and $\reals^\N$ are the same, as are those of $\N$ and $\Q$, and hence there is a bijection $\varphi: \reals \to \reals^\Q$ \citep[Theorem 2.3]{hrbacek1999introduction}.
  Taking $\Y=\reals$,
any probability measure defined on the Borel $\sigma$-algebra is uniquely determined by its cumulative distribution function (CDF) $F$ which is in turn uniquely determined by its values on the rationals $\{F(q)|q\in\Q\}$.
  Let $g:\P\to\reals^\Q$ be the map which converts probability measure $p$ to its CDF and evaluates it on the rationals.
  Then $h \defeq \varphi^{-1} \circ g$ is an injective map between $\P$ and $\reals$.
  Thus, given some property $\Gamma:\P\to\reals^k$, we let $\hat\Gamma = h$ encode each distribution into a single real number.
  We elicit $\Gamma$ with $L(\hat r,y) = L^*(h^{-1}(\hat r),y)$ for some proper loss function $L^*:\P\times\Y\to\reals$ which elicits entire distributions \citep{gneiting2007strictly}, and finally take $f = \Gamma \circ h^{-1}$ so that $f \circ \hat\Gamma = \Gamma \circ h^{-1} \circ h = \Gamma$.
  We conclude that if $\C=\EL(\P)$ is the set of all elicitable properties, then $\elic_\C(\Gamma)=1$ for all properties $\Gamma$.
\end{remark}

Behind essentially all of our nontrivial lower bounds is the concept of identification complexity.
\begin{definition}
  \label{def:iden-complex}
  A property $\Gamma$ is \myemph{$k$-identifiable},
  $k\in\N\cup\{\countinf\}$,
  if there exists $\hat \Gamma \in \ID_k(\P)$ and $f$ such that $\Gamma = f \circ \hat \Gamma$.
  The \myemph{identification complexity} of $\Gamma$ is $\iden(\Gamma) = \min\{k: \Gamma \text{ is $k$-identifiable}\}$.
\end{definition}
From our definitions, $\elici(\Gamma)\geq\iden(\Gamma)$ when both are defined, since the property $\hat \Gamma$ which in Definition~\ref{def:elic-complex} must be identifiable for $\C=\I$. 
In particular, Condition~\ref{cond:v-interior} already implies an identification complexity lower bound, which in turn lower bounds elicitation complexity.
\begin{lemma}
  \label{lem:main-iden-lower-bound}
  Let $\Gamma\in\I_k(\P)$ satisfy Condition~\ref{cond:v-interior} for some $r\in\Gamma(\P)$.
  Then $\iden(\Gamma)\geq k$.
\end{lemma}
To illustrate Definition~\ref{def:iden-complex}, recall the variance example, where $\Gamma = \Var$, $\hat\Gamma : p \mapsto (\E_p[Y],\E_p[Y^2]) \in \reals^2$, and $f:(r_1,r_2) \mapsto r_2 - r_1^2$.
Here $\iden(\Var) \leq 2$, via $V(r,y) = (y-r_1,y^2-r_2)$.
Of course, as $\hat\Gamma \in \Clin$, we also have the stronger statement $\elic_\Clin(\Var) \leq 2$.  
By Lemma~\ref{lem:main-iden-lower-bound}, we also have $\iden(\hat\Gamma) = 2$ for suitably rich $\P$.
As we now show, this can be used to provide a lower bound that $\iden(\Var) = 2$ as well.

\subsection{Redundancy and Refinement}
\label{sec:redund-refin}

It is easy to create redundant properties in various ways.
For example, given elicitable properties $\Gamma_1$ and $\Gamma_2$ the property $\Gamma \defeq \{\Gamma_1,\Gamma_2,\Gamma_1+\Gamma_2\}$ clearly contains redundant information.
We will use curly braces to combine properties when the order is irrelevant.
A concrete case is $\Gamma=\{$mean squared, variance, 2nd moment$\}$, which, as we have seen, has $\elici(\Gamma)\leq 2$.
Adding properties to such a list cannot lower its overall complexity, however, and cannot increase it beyond the sum of the individual complexities either; i.e., elicitation complexity is sub-additive.

\begin{lemma}
  \label{lem:complex-sub-add}
  For all properties $\Gamma_1,\ldots,\Gamma_m$, and classes $\C$, we have \[\max_{1 \leq i \leq m} \elic_\C(\Gamma_i) \leq \elic_\C(\{\Gamma_1,\ldots,\Gamma_m\}) \leq \sum_{i=1}^m \elic_\C(\Gamma_i)~.\]
\end{lemma}
\begin{proof}
  For the first inequality, letting $k = \elic_\C(\{\Gamma_1,\ldots,\Gamma_m\}) \in \N\cup\{\infty\}$, we have an elicitable $\hat\Gamma \in \C$, $\hat\Gamma:\P\to\reals^k$, and $f$ such that $(\Gamma_1,\ldots,\Gamma_m) = f \circ \hat\Gamma$.
  Letting $g$ be the projection which picks out the $i$th coordinate, we have $\Gamma_i = (g\circ f) \circ \hat\Gamma$, thus establishing $\elic(\Gamma_i) \leq k$.
  For the second, for any elicitable $\hat\Gamma_i \in \C$ and $f_i$ with $\Gamma_i = f_i \circ \hat\Gamma_i$, we of course can take $\hat\Gamma = (\hat\Gamma_1,\ldots,\hat\Gamma_m)$ and $f = (f_1,\ldots,f_m)$ so that $(\Gamma_1,\ldots,\Gamma_m) = f \circ \hat\Gamma$.
\end{proof}

The following definitions and lemma capture various aspects of a lack of redundancy, which together ensure that the second inequality of Lemma~\ref{lem:complex-sub-add} will be tight.

\begin{definition}
  \label{def:full-rank}
  Property $\Gamma:\P\to\reals^k$ in $\ID(\P)$ is \myemph{balanced} if $\iden(\Gamma)=k$.
\end{definition}

There are two ways for a property to fail to be balanced.
First, as the examples above suggest, $\Gamma$ can be ``redundant'' so that it is a link of a lower-dimensional identifiable property.
Balance can also be violated if \myemph{more} dimensions are needed to identify the property than to specify it.
This is the case with most of the properties in \S~\ref{sec:elic-complex-exampl-appl}, e.g., the variance which is a 1-dimensional property but which we will show has $\iden(\Var)=2$.

\begin{definition}
  \label{def:prop-indep}
  Properties $\Gamma,\Gamma'\in\ID(\P)$ are \myemph{independent} if
  $\iden(\{\Gamma,\Gamma'\}) = \iden(\Gamma)+\iden(\Gamma')$.
\end{definition}

\begin{lemma}
  \label{lem:elic-complex-independent}
  If $\Gamma,\Gamma'\in\EL(\P)\cap\ID(\P)$ are independent and balanced, then we have $\elici(\{\Gamma,\Gamma'\}) = \elici(\Gamma) + \elici(\Gamma')$.
\end{lemma}

\begin{proof}
  Let $\Gamma:\P\to\reals^k$ and $\Gamma':\P\to\reals^{k'}$.
  As $\Gamma,\Gamma'\in\EL(\P)\cap\ID(\P)$, we have $\elici(\Gamma) \leq k$ and $\elici(\Gamma') \leq k'$.
  Unfolding our definitions, we have $\elici(\{\Gamma,\Gamma'\}) \geq \iden(\{\Gamma,\Gamma'\}) = \iden(\Gamma) + \iden(\Gamma') = k + k' \geq \elici(\Gamma) + \elici(\Gamma')$.  For the upper bound, we simply take losses $L$ and $L'$ for $\Gamma$ and $\Gamma'$, respectively, and elicit $\{\Gamma,\Gamma'\}$ via $\hat L(r,r',y) = L(r,y) + L'(r',y)$.
\end{proof}

To illustrate the lemma, $\elici(\Var)=2$, yet $\Gamma=\{\E[Y],\Var\}$ has $\elici(\Gamma)=2$, so clearly the mean and variance are not both independent and balanced.
As we have remarked, variance is not balanced.
However, the mean and second moment satisfy both by Lemma~\ref{lem:elic-complex-means}.

Similar to redundancy, we can think of one property \myemph{refining} another, in the sense of encoding strictly more information.
\begin{definition}
  \label{def:refine}
  $\Gamma'$ \myemph{refines} $\Gamma$ if there exists a function $f$ such that $\Gamma = f \circ \Gamma'$.
\end{definition}
Equivalently, $\Gamma'$ refines $\Gamma$ if each level set of $\Gamma'$ is contained in a level set of $\Gamma$.
Immediately, a property which refines another cannot have lower elicitation complexity.

\begin{lemma}
  \label{lem:refine}
  If $\Gamma'$ refines $\Gamma$ then  $ \elic_\C(\{\Gamma,\Gamma'\}) = \elic_\C(\Gamma') \geq \elic_\C(\Gamma)$.
\end{lemma}

\begin{proof}
For the inequality, if $\Gamma'$ is $k$-elicitable with respect to $\C$, then there exists an elicitable $\hat \Gamma \in \C$ such that $\Gamma' = g \circ \hat \Gamma$.  But then $\Gamma = f \circ g \circ \hat \Gamma$, so $\Gamma$ is also $k$-elicitable with respect to $\C$.
For the equality, $ \elic_\C(\{\Gamma,\Gamma'\}) \geq \elic_\C(\Gamma')$ follows by Lemma~\ref{lem:complex-sub-add}.  To see that we also have  $\elic_\C(\{\Gamma,\Gamma'\}) \leq \elic_\C(\Gamma')$, observe that $\{\Gamma,\Gamma'\} = \{f \circ g,g\} \circ \hat \Gamma$.
\end{proof}

With this observation about refinement, we can finally conclude that $\iden(\Var) = 2$, because the pair of the mean and second moment refines the variance.  In fact the reverse is true as well because the mapping is a bijection.  In this sense our lower bounds care only about the geometry of the level sets of $\Gamma$, not on how those are labeled.

\subsection{Upper Bounds}
\label{sec:basic-upper-bounds}

We now provide some straightforward upper bounds which hold for every property.
Clearly, whenever $p \in \P$ can be uniquely determined by some number of elicitable parameters then the elicitation complexity of every property is at most that number: one can simply elicit the entire distribution and then the link function simply computes the desired property.
The following propositions give two notable applications of this observation.
We adopt the convention that $F$ denotes a cumulative distribution function (CDF).
Recall that we denote a countably infinite elicitation complexity by $\countinf$.

\begin{proposition}
  \label{prop:everything-elic-n-1}
  When $|\Y| \in \N$, every property $\Gamma$ has $\elic_\C(\Gamma)\leq |\Y|-1$ for all $\C \supseteq \Clin$.
\end{proposition}
\begin{proof}
  Letting $\Y = \{y_1,\ldots,y_n\}$, a distribution $p$ is uniquely determined by its first $n-1$ components $p(y_1),\ldots,p(y_{n-1})$, each of which are elicitable as linear properties $p(y) = \E_p\ones_{Y=y}$.
\end{proof}

\begin{proposition}
  \label{prop:elic-complex-countable-cdf}
  When $\Y=\reals$, every property $\Gamma$ has $\elic_\C(\Gamma)\leq\countinf$ for all $\C \supseteq \Clin$.
\end{proposition}

\begin{proof}
  Since a distribution is determined by the values of its CDF $F$ on a dense set,  let $\{q_i\}_{i\in\N}$ be an enumeration of the rational numbers, and define $\hat\Gamma(F)_i = 2^{-i} F(q_i)$.
  As $\hat\Gamma$ is square-summable, we have $\hat\Gamma \in \Clin$, cf.\ discussion with the proof of Proposition~\ref{prop:nested-convex-classes-long}, and elicited by $L(\{r_i\}_{i\in\N},y) = \sum_{i\in\N} (r_i-2^{-i} \ones_{y\leq q_i})^2$.
  With an appropriate link we can compute $\Gamma$.
\end{proof}

The restrictions above on $\Y$ may easily be placed on $\P$ instead.
For example, finite $\Y$ is equivalent to $\P$ having support on a finite subset of $\Y=\reals$.

In particular, Proposition~\ref{prop:everything-elic-n-1} and~\ref{prop:elic-complex-countable-cdf} apply to the identity property $\idprop(p) = p$,
and as we now show with a turn to lower bounds, under mild conditions the bounds they give are tight for it.

\subsection{Lower bounds for specific properties: expectations and quantiles}
\label{sec:basic-means-quantiles}

A well-studied class of properties is the set of expectations of some vector-valued random variable, often called the \myemph{linear} case.
All such properties are elicitable and identifiable \citep{savage1971elicitation,abernethy2012characterization,frongillo2015vector}, with complexity bounded by the dimension of the random variable, but of course the complexity can be lower if the range of $\Gamma$ is not full-dimensional.
In what follows, let $\affdim$ denote the dimension of the affine hull.

\begin{lemma}
  \label{lem:elic-complex-means}
  Let $\phi:\Y\to\reals^k$ be $\P$-integrable, $k\in\N$, and let $\Gamma(p)=\E_p[\phi(Y)]$.
  Then $\elic_\C(\Gamma)=\affdim\,\Gamma(\P)$, the dimension of the affine hull of the range of $\Gamma$, for any $\C$ satisfying $\Clin \subseteq \C \subseteq \I$.
\end{lemma}

Quantiles are another important case:
for sufficiently rich sets of distributions, distinct quantiles are independent and balanced, so their elicitation complexity is the number of quantiles being elicited.
Here we take $\C=\I$ as losses eliciting quantiles cannot be strictly convex; see \S~\ref{sec:ex-es}.
As with expectations, if the set of distributions is not sufficiently rich the elicitation complexity can be lowered.
We state two versions of the condition that $\P$ be ``rich''.
These conditions are satisfied by, for example, the set of all mixtures of univariate Gaussian distributions.

\begin{condition}
  \label{cond:quantiles}
  \label{cond:whole-dist}
  Let $k\in\N$ be given.
  For all $x\in[0,1]^k$, there exist $r_1,\ldots,r_k\in\reals$ such that $x\in\interior\{(F(r_1),\ldots,F(r_k))^\tr  : F \in \P \} \subseteq \reals^k$.
\end{condition}

\begin{condition}
  \label{cond:linear}
  Let $k\in\N$ be given.
  There exists a $\P$-integrable function $\phi:\Y\to\reals^k$ with $\affdim\,\{\E_p[\phi(Y)] : p\in\P\} = k$.
\end{condition}

Both of these conditions can be throught of as special cases of applying Condition~\ref{cond:v-interior} for various choices of $r$ to the identification function for the $\alpha$-quantile $V(r,y) = \ones_{y\leq r} - \alpha$, or $V(r,p) = F(r) - \alpha$.
Again, Condition~\ref{cond:linear} is implied by Assumption V1 of Fissler and Ziegel~\citeyearpar{fissler2016higher}.
Condition~\ref{cond:quantiles} implies Condition~\ref{cond:linear}, by considering $\phi(y)_i = \ones_{y\leq r_i}$.

\begin{lemma}
  \label{lem:elic-complex-quantiles}
  For $\Y=\reals$ let $\P\subseteq\Pquant$, defined in \S~\ref{sec:ex-es},
  satisfy Condition~\ref{cond:quantiles} for some $k\in\N$.
  For all distinct $\alpha_1,\ldots,\alpha_k \in (0,1)$, we have $\elici(\{q_{\alpha_1},\ldots,q_{\alpha_k}\})=k$,
  where $q_{\alpha}$ is the $\alpha$-quantile function.
\end{lemma}

The quantile example in particular allows us to see that all complexity classes, including $\countinf$, are occupied.
In fact, from the examples in \S~\ref{sec:ex-cvx-ftns-of-means}, we can see that even for \myemph{real-valued} properties $\Gamma:\P\to\reals$, all classes are occupied.
Recall that Condition~\ref{cond:quantiles} implies Condition~\ref{cond:linear}.

\begin{proposition}
  \label{prop:elic-complex-classes}
  Let $\P$ satisfy Condition~\ref{cond:linear}, or Condition~\ref{cond:quantiles}, for all $k\in\N$.
  Then for all $k\in\N\cup\{\infty\}$ there exists a property $\gamma_k:\P\to\reals$ with
  $\elic_\C(\gamma_k)=k$ for any $\C$ satisfying $\Clin \subseteq \C \subseteq \I$.
\end{proposition}
\begin{proof}
  Letting $\phi:\Y\to\reals^k$ be the random variable from Condition~\ref{cond:linear}, we may take $\gamma_k(p) = \|\E_p[\phi(Y)]\|^2$ by 
  Corollary~\ref{cor:func-mean}.
  The case $k=\infty$ follows from Corollary~\ref{cor:entropy-norm}.
\end{proof}

We now give a matching lower bound to Propositions~\ref{prop:everything-elic-n-1} and~\ref{prop:elic-complex-countable-cdf}, stating that the complexity of eliciting the whole distribution via identifiable properties is maximal when $\P$ is sufficiently rich.
This observation constrasts with Remark~\ref{remark:everything-1-elic}, where we saw that $\elic_\C(\Gamma)=1$ when $\C$ is too large.

\begin{lemma}
  \label{lem:complex-whole-dist}
  Let $\idprop:\P\to\P$, $\idprop:p\mapsto p$.
  The following hold for all $\Clin \subseteq \C \subseteq \I$.
  If $\Y$ is finite, then $\elic_\C(\idprop)=\affdim\, \P$; in particular, if $\P$ is the probability simplex, then $\elic_\C(\idprop)=|\Y|-1$.
  If $\Y=\reals$ and there are infintely many $k \in \N$ satisfying Condition~\ref{cond:whole-dist} or~\ref{cond:linear}, then $\elic_\C(\idprop)=\countinf$.
\end{lemma}

\begin{proof}
  For $|\Y|<\infty$, observe that $\idprop$ is linear and apply Lemma~\ref{lem:elic-complex-means}.
  For $\Y=\reals$, given finite $k$ such that $\P$ satisfies Condition~\ref{cond:quantiles} or~\ref{cond:linear}, define $\gamma_k$ as in the proof of Proposition~\ref{prop:elic-complex-classes};
  as $\idprop$ refines all properties, Lemma~\ref{lem:refine} gives $\elic_C(\idprop) \geq \elic_\C(\gamma_k) = k$.
  We now have $\elic_C(\idprop) \geq k$ for infinitely many $k\in\N$, and $\elic_C(\idprop) \leq \countinf$ from Proposition~\ref{prop:elic-complex-countable-cdf}.
\end{proof}

\section{Eliciting the Bayes Risk}
\label{sec:elic-complex-bayes-risk}

\subsection{Upper Bound}
\label{sec:bayes-risk-upper-bound}

For the upper bound, we construct losses explicitly for properties that can be expressed as the pointwise minimum of an indexed set of random variables $\{X_a\}_{a\in\A}$,
\begin{equation}
  \label{eq:prop-minimum}
  \gamma:\P\to\reals,\quad \gamma(p) = \min_{a\in\A} \E_p[X_a]~.
\end{equation}
An important special case, of course, are Bayes risks.
Recall that the Bayes risk of a loss function $L:\A\times\Y \to \reals$ is defined as $\lbar(p) := \inf_{a \in \A} L(a,p)$.
Interestingly, our construction does not elicit the minimum directly, but as a joint elicitation of the minimum value and the index that realizes this value.
The loss function takes the form of a loss eliciting the linear property $p\mapsto \E_p[X_a]$, except that here the index $a$ is not fixed, but elicited as well.

\begin{theorem}
  \label{thm:elic-minimum}
  Let $\{X_a\}_{a\in\A}$ be a set of $\P$-integrable random variables indexed by $\A\subseteq\reals^k$, $k\in\N\cup\{\infty\}$.  If $\inf_a \E_p[X_a]$ is attained for all $p\in\P$, then the loss function
  \begin{equation}
    \label{eq:loss-elic-minumum}
    L((r,a),y) = H(r) + h(r)(X_a(y)-r)
  \end{equation}
  elicits the set-valued property $\hat\Gamma : p \mapsto \{(\gamma(p),a) : \E_p[X_a]\!=\!\gamma(p)\}$, where $\gamma$ is defined in~\eqref{eq:prop-minimum}, $h:\gamma(\P)\to\reals_+$ is any strictly decreasing function, and $H(r) = \int_{r_0}^r h(x) dx$ for some $r_0 \in \gamma(\P)$.
\end{theorem}
\begin{proof}
  Working with gains instead of losses, we will show the equivalent result that
  $S((r,a),y) = g(r) + dg_{r}(X_{a} - r)$ elicits the combined property $\hat\Gamma : p \mapsto \{(\gamma(p),a) : \E_p[X_a]=\gamma(p) \}$ for $\gamma(p) = \max_a \E_p[X_a]$.
  Here $g$ is a convex function with a strictly increasing and positive subgradient $dg$.

  For any fixed $a$, we have by the subgradient inequality,
  \[ S((r,a),p) = g(r) + dg_{r}(\E_p[X_{a}] - r) \leq g(\E_p[X_{a}]) = S((\E_p[X_{a}],a),p)~, \]
  and as $dg$ is strictly increasing, $g$ is strictly convex, so $r=\E_p[X_{a}]$ is the unique maximizer.
  Now letting $\tilde S(a,p) = S((\E_p[X_{a}],a),p)$, we have
  \[ \argmax_{a\in\A} \tilde S(a,p) = \argmax_{a\in\A} g(\E_p[X_{a}]) = \argmax_{a\in\A} \E_p[X_{a}]~, \]
  because $g$ is strictly increasing.  We now have
  \begin{align*}
    \argmax_{a\in\A,r\in\reals} S((r,a),p)
    &= \Bigl\{(\E_p[X_{a}],a):a\in\argmax_{a\in\A} \E_p[X_{a}]\Bigr\}~.
  \end{align*}
\end{proof}

We briefly mention various forms of Theorem~\ref{thm:elic-minimum} which have appeared in the literature.
Most recently, a similar result appears independently in the Master's thesis of Jonas Brehmer~\citeyearpar{brehmer2017elicitability}.
The loss function of Fissler and Ziegel~\citeyearpar{fissler2016higher} for expected shortfall is a special case of Theorem~\ref{thm:elic-minimum}, and indeed a careful inspection of the former gave the inspiration for the latter.
Earlier work of Peter Gr\"unwald~\citeyearpar{grunwald1999viewing,grunwald2008that} gives a version of Theorem~\ref{thm:elic-minimum} in the context of the minimum description length principle; here the description length is defined in terms of a given loss function and a parameter $\beta$, and for certain ``simple'' classes of losses, the $\beta$ value minimizing the description length is precisely the Bayes risk of the given loss.
Finally, concurrent to our work, Fissler and Ziegel~\citeyearpar{fissler2019elicitability} give a construction for Range Value at Risk, which motivates a more general construction for linear combinations of minimum expectations in the form~\eqref{eq:prop-minimum}; see \S~\ref{sec:ex-es}.

Proving the upper bound in our main theorem, that the Bayes risk of a loss eliciting a $k$-dimensional property is itself $(k+1)$-elicitable, is a straightforward corollary of Theorem~\ref{thm:elic-minimum}.
Specifically, given a loss $L:\reals^k\times\Y \to \reals$ eliciting $\Gamma:\P\to\reals^k$, we simply let $X_a = L(a,Y)$ so that the pointwise minimum becomes the Bayes risk $\gamma(p) = \lbar(p)$; Theorem~\ref{thm:elic-minimum} then states that, as long as $(\lbar,\Gamma) \in \C$, we have $\elic_\C(\lbar) \leq k+1$.
The infimum in the definition of the Bayes risk is attained as $L$ elicits $\Gamma$.

\begin{corollary}
  \label{cor:elic-complex-bayes-risk}
  If $L:\reals^k\times\Y \to \reals$ elicits $\Gamma:\P\to\reals^k$, $k\in\N\cup\{\infty\}$, then the loss
  \begin{equation}
    \label{eq:elic-bayes-risk-cor}
    L^*((r,a),y) = L'(a,y) + H(r) + h(r)(L(a,y) - r)
  \end{equation}
  elicits $\{\lbar,\Gamma\}$, where $h:\reals\to\reals_+$ is any positive strictly decreasing function, $H(r) = \int_0^r h(x) dx$, and $L'$ is any other loss weakly eliciting $\Gamma$.
  If $(\lbar,\Gamma)\in\C$, $\elic_\C(\lbar) \leq k+1$.
\end{corollary}

To illustrate the upper bound, let us return to the variance example.
Take $X_a = (Y-a)^2$ to be squared loss, so that $\gamma(p) = \min_a \E_p[(Y-a)^2]$, and because squared loss is minimized by the mean $a = \E_p[Y]$, we have $\gamma(p) = \E_p[(Y-\E_p[Y])^2] = \Var(p)$.
Theorem~\ref{thm:elic-minimum} therefore states that $\hat\Gamma : p \mapsto (\Var(p),\E_p[Y])$ is elicitable.
Corollary~\ref{cor:elic-complex-bayes-risk} is more direct: as squared loss $L(r,y) = (r-y)^2$ elicits the mean, and $\lbar(p) = \Var(p)$, for any class of properties $\C$ where $(\Var,\E[Y])\in\C$ we have $\elic_\C(\Var) \leq 2$.
Interestingly, we do not have $(\Var,\E[Y])\in\Clin$, but as described in \S~\ref{sec:ex-variance}, the upper bound for $\Clin$ still holds by way of the first two moments.
In that section we also illustrate that
Theorem~\ref{thm:elic-minimum} does not characterize all possible loss functions to elicit the joint property $\hat\Gamma$.

\subsection{Lower Bound}
\label{sec:bayes-risk-lower-bound}

We now turn to lower bounds.
A first observation is that $\lbar$ is concave, and thus unlikely to be elicitable directly, as the level sets of $\lbar$ are likely to be non-convex.
To show a lower bound greater than 1, however, we will need much stronger techniques.
In particular, while $\lbar$ must be concave, it may not be strictly so.
Indeed, $\lbar$ must be flat between any two distributions which share a minimizer.
Crucial to our lower bound is the fact that whenever the minimizer of $L$ \myemph{differs} between two distributions, $\lbar$ is essentially strictly concave between them.

\begin{lemma}
  \label{lem:elic-complex-bayes-concave}
  \label{LEM:ELIC-COMPLEX-BAYES-CONCAVE} 
  Suppose the loss $L$ with Bayes risk $\lbar$ elicits $\Gamma:\P\to\R$.  Then for any $p,p'\in\P$ with $\Gamma(p)\neq\Gamma(p')$, we have $\lbar(\lambda p + (1-\lambda) p') > \lambda \lbar(p) + (1-\lambda) \lbar(p')$ for all $\lambda\in(0,1)$.
\end{lemma}

We can now prove our main lower bound, that the Bayes risk of a loss eliciting $\Gamma$ has complexity at least that of $\Gamma$.
The argument proceeds by showing that if we elicit the Bayes risk indirectly through some $\hat\Gamma$, then $\hat\Gamma$ must refine $\Gamma$ by Lemma~\ref{lem:elic-complex-bayes-concave}, from which the result follows.

\begin{theorem}
  \label{thm:bayes-risk-lower-bound}
  Let class of properties $\C$ be given.
  If $L$ elicits $\Gamma$, and $\elic_\C(\lbar)$ is defined, then $\elic_\C(\lbar) \geq \elic_\C(\Gamma)$,
  with equality if $\lbar=f\circ\Gamma$ for some function $f$.
\end{theorem}
\begin{proof}
  Let $\ell=\elic_\C(\lbar)$, so that we have some $\hat\Gamma\in\EL_\ell\cap\C$ and $g:\reals^\ell\to\reals$ such that $\lbar = g \circ \hat \Gamma$.
  We show by contradiction that $\hat\Gamma$ refines $\Gamma$.
  Otherwise, we have $p,p'$ with $\hat\Gamma(p)=\hat\Gamma(p')$, and thus $\lbar(p) = \lbar(p')$, but $\Gamma(p) \neq \Gamma(p')$.
  Lemma~\ref{lem:elic-complex-bayes-concave} would then give us some $p_\lambda = \lambda p + (1-\lambda) p'$ with $\lbar(p_\lambda) > \lbar(p)$,
  but as the level sets $\hat\Gamma_{\hat r}$ are convex by Proposition~\ref{prop:elic-complex-level-sets-convex}, we would have $\hat\Gamma(p_\lambda) = \hat\Gamma(p)$, which would imply $\lbar(p_\lambda)=\lbar(p)$.
  Thus, $\hat\Gamma$ must refine $\Gamma$, so by Lemma~\ref{lem:refine}, $\elic_\C(\lbar) = \ell \geq \elic_C(\hat\Gamma) \geq \elic_\C(\Gamma)$.
  If $\lbar=f\circ\Gamma$ then $\Gamma$ refines $\lbar$, so we also have $\elic_C(\Gamma) \geq \elic_\C(\lbar)$.
\end{proof}

We now restate and prove our main theorem.

\tempsetcounter{theorem}{0}
\begin{theorem}
  Let $L:\reals^k\times\Y \to \reals$ be a loss function eliciting $\Gamma:\P\to\reals^k$, $k\in\N\cup\{\infty\}$, and $\lbar$ be its Bayes risk.
  If $(\lbar,\Gamma)\in \C$ and $\elic_\C(\Gamma)=k$, then $\elic_\C(\lbar) \in \{k,k+1\}$.
  Moreover, the loss
  \begin{equation*}
    L^*((r,a),y) = L'(a,y) + H(r) + h(r)(L(a,y) - r)
  \end{equation*}
  elicits $\{\lbar,\Gamma\}$, where $h:\reals\to\reals_+$ is any positive strictly decreasing function, $H(r) = \int_0^r h(x) dx$, and $L'$ is any other loss weakly eliciting $\Gamma$.
\end{theorem}
\restorecounter{theorem}
\begin{proof}
  Corollary~\ref{cor:elic-complex-bayes-risk} gives form of the loss and the upper bound $\elic_\C(\lbar) \leq k+1$.
  For lower bound, Theorem~\ref{thm:bayes-risk-lower-bound} together with the assumption $\elic_\C(\Gamma)=k$ gives $\elic_\C(\lbar) \geq \elic_\C(\Gamma) = k$.
\end{proof}

\subsection{Bounds for Specific Property Classes}
\label{sec:bayes-risk-specific-classes}

We now turn to results for specific choices of the class $\C$.
To begin, Proposition~\ref{prop:tighter-lower-bound} gives tighter lower bounds when $\C\subseteq\I$, the weakest of the classes we consider.
This specialization is useful; often the most difficult requirement of Theorem~\ref{thm:bayes-risk-lower-bound} is to show $\elic_\C(\Gamma)=k$, but this is implied by Condition~\ref{cond:v-interior} when $\C\subseteq\I$.
To further tighten the lower bound to $\elic_\C(\lbar) \geq \elic_\C(\Gamma)+1$, we essentially must rule out the case where $\lbar$ is a link of $\Gamma$.
This case does arise; for example, dropping the $y^2$ term from squared loss gives $L(x,y) = x^2-2xy$ and $\lbar(p)=-\E_p[Y]^2$, which yields $\elic_\C(\lbar)=1$ for any reasonable choice of $\C$, e.g., $\C=\I$.
To rule out this case, we assume that $\lbar$ is not constant on some level set $\Gamma_r$ which satisfies Condition~\ref{cond:v-interior}.
The proof then argues that if $\elici(\lbar) = \elici(\Gamma)$, some level set of $\lbar$ must contain $\Gamma_r$, a contradiction.
It also argues that we may replace the condition $(\lbar,\Gamma)\in\C$ by $\Gamma\in\I$.

\begin{corollary}
  \label{cor:bayes-risk-ident-lower-bound}
  Let $L$ elicit some $\Gamma\in\ID_k(\P)$, $k\in\N$.
  If $\Gamma$ refines $\lbar$, then $\elici(\lbar) = k$.
  If $\Gamma$ satisfies Condition~\ref{cond:v-interior} for some $r\in\Gamma(\P)$ and $\lbar$ is non-constant on $\Gamma_r$, then $\elic_\I(\lbar) = k+1$.
\end{corollary}

We now restate and prove Proposition~\ref{prop:tighter-lower-bound}, which we used extensively in our applications.
\tempsetcounter{proposition}{2}
\begin{proposition}
  Let $L:\reals^k\times\Y \to \reals$ be a loss eliciting $\Gamma \in \I_k$, $k\in\N$.
  If $\Gamma$ satisfies Condition~\ref{cond:v-interior} for some $r\in\Gamma(\P)$, and $\lbar$ is non-constant on $\Gamma_r$, then $\elici(\lbar) = k+1$.
  If additionally $(\lbar,\Gamma)\in\C$ for some $\C \subseteq \I$, then $\elic_\C(\lbar) = k+1$.
\end{proposition}
\restorecounter{proposition}
\begin{proof}
  The second statement follows from Corollary~\ref{cor:bayes-risk-ident-lower-bound}.
  The third follows from Theorem~\ref{thm:bayes-risk-lower-bound} together with Proposition~\ref{prop:nested-convex-classes-long} giving $\elic_\C(\lbar) \geq \elici(\lbar)$.
\end{proof}

We now turn to upper and lower bounds for strictly and strongly convex losses.  We provide the full treatment in the supplemental material.  Here we state our main conclusion for strongly convex losses; the result for strict convexity is similar but requires some additional assumptions.

\begin{proposition}
  \label{prop:str-cvx-elic}
  Let $\Gamma\in\Cstrong$, $\Gamma:\P\to\reals^k$, $k\in\N$, be elicited by a differentiable, bounded, strongly convex $L$.
 If $\Gamma$ satisfies Condition~\ref{cond:v-interior} for some $r\in\Gamma(\P)$, and $\lbar$ is non-constant on $\Gamma_r$, then $\elic_\Cstrong(\lbar) = k+1$.
\end{proposition}

\section{Discussion and Open Questions}
\label{sec:discussion}

As discussed above, our notion of elictiation complexity, Definition~\ref{def:elic-complex}, builds on \citet{lambert2008eliciting} among other work.
We believe our definition is best suited to studying the difficulty of eliciting properties:
viewing $f$ as a potentially dimension-reducing link function, our definition captures the minimum number of dimensions needed in a point estimation or empirical risk minimization for the property in question, followed by a simple one-time application of $f$.
For a comparison to other definitions in the literature and further discussion, see \S~\ref{sec:additional}.

\label{sec:open-problems}
Many natural problems in elicitation complexity remain open.
Most apparent are the characterizations of the complexity classes $\{\Gamma : \elic_\C(\Gamma) = k\}$, and in particular, determining the elicitation complexity of non-elicitable properties.
For example, subsequent to our work, the complexity of the mode is shown to be infinite \citep{dearborn2019indirect}, while that of the smallest prediction interval remains open \citep{frongillo2014general}.
We identify other future directions below.

\paragraph{Tighter characterization for Bayes risks}
Consider a loss $L$ eliciting some property $\Gamma$ of elicitation complexity $k$.
Intuitively, Corollary~\ref{cor:bayes-risk-ident-lower-bound} says the elicitation complexity of the Bayes risk $\lbar$ is $k+1$, unless $\lbar$ happens to be a link of $\Gamma$.
Yet we lack a characterization of properties $\Gamma$ for which $\lbar = f \circ \Gamma$ for some link $f$ and some $L$ eliciting $\Gamma$.
We conjecture that this relationship is only possible if $\Gamma$ is link of a linear property, i.e.,\ $\Gamma(p) = \varphi(\E_p[g(Y)])$ for some invertible $\varphi$ and arbitrary $g$.
As intuition, $\lbar(p)$ must have slope zero along level sets of $\Gamma$.

\paragraph{General convex losses}
Throughout the paper, when working with convex losses, we have insisted that they be smooth and strictly convex.
An important future direction is to study the natural class $\Ccvx$ of properties elicited by any convex loss.
Our results do not apply to this class, as fundamentally our lower bounds rely on identifiability, i.e., $\C\subseteq\I$, whereas $\Ccvx \not\subseteq \I$.
Remark~\ref{remark:everything-1-elic} shows that the class $\Ccvx$ is restrictive enough to prevent $\elic_{\Ccvx}(\Gamma) = 1$ for all properties $\Gamma$ \citep{ramaswamy2013convex}.
While some results for $\elic_\Ccvx$ have appeared in the machine learning literature, for settings such as classification or ranking \citep{bartlett2006convexity,ramaswamy2013convex} and some more general results under the name \myemph{convex calibration dimension} \citep{ramaswamy2016convex,agarwal2015consistent}, tight bounds remain elusive in general.

\paragraph{Conditional elicitation}
Another interesting direction is conditional elicitation: properties which are elicitable as long as the value of some other elicitable property is known.
This notion was introduced by Emmer et al.~\citeyearpar{emmer2015what}, who showed that the variance and expected shortfall are both conditionally elicitable, on the mean $\E_p[Y]$ and quantile $q_\alpha(p)$, respectively.
Intuitively, knowing that $\Gamma$ is elicitable conditional on an elicitable $\Gamma'$ would suggest that perhaps the pair $\{\Gamma,\Gamma'\}$ is elicitable; Fissler and Ziegel~\citeyearpar{fissler2016higher} It is an open question whether and when this joint elicitability holds in general.
From our results, we now see a broad class of properties for which this joint elicitability does hold: the Bayes risk $\lbar$, of a loss $L$ eliciting $\Gamma$, is elicitable conditioned on $\Gamma$, and the pair $\{\Gamma,\lbar\}$ is jointly elicitable from Theorem~\ref{thm:elic-minimum}.
We give a counter-example in Figure~\ref{fig:3d-examples}, however, with a property which is conditionally elicitable but not jointly.

\subsection*{Acknowledgements}
We would like to thank
Yiling Chen,
Krisztina Dearborn,
Jessie Finocchiaro,
Tobias Fissler,
Tilmann Gneiting,
Peter Gr\"unwald,
Nicolas Lambert,
Ingo Steinwart,
Bo Waggoner,
Ruodu Wang,
Jens Witkowski,
and
Johanna Ziegel,
for helpful comments, discussions, and references.
We thank anonymous reviewers for the insights in Remark~\ref{remark:spectral-risk-bound} and on pairs of properties in Lemma~\ref{lem:refine}.
This work was funded in part by National Science Foundation Grant CCF-1657598.

\newpage
\appendix

\section{Proof of Proposition~\ref{prop:nested-convex-classes-long}}
\label{sec:proof-nested}

In the case of $k = \infty$, we interpret the $\reals^\infty$ in the statement of the proposition as the sequence space $\ell^2$, and require Fr\'echet differentiability in $\Cstrict$.
The restriction that properties in our four classes must take on values which are square-summable is important for the loss $L$ in the proof, e.g.\ to have $\Clin \subseteq \Cstrong$.

\begin{proof}
  Let $\Gamma\in\Clin$, so that $\Gamma(p) = \E_p[\phi(Y)]$ for some $\phi$.
  Taking the loss $L(r,y) = \|r\|^2 - 2 r \cdot \phi(y)$, which is differentiable and Lipschitz continuous on the (assumed bounded) domain of $r$, and furthermore strongly convex with constant $\mu = 2$, showing $\Gamma \in \Cstrong$.
  The inclusion $\Cstrong\subseteq\Cstrict$ is immediate from the definition.
  Finally, let $L(r,y)$ be a differentiable, Lipschitz-continuous, strictly convex loss function eliciting $\Gamma$.
  Letting $V(r,y) = \nabla_r L(r,y)$, we have $\Gamma(p) = r \implies \nabla_r \E_p L(r,Y) = 0$.
  As $L$ is Lipschitz continuous, the dominated convergence theorem gives us $\nabla_r \E_p L(r,Y) = 0 \iff \E_p \nabla_r L(r,Y) = 0$.
  Conversely, as $\E_p L(r,Y)$ is strictly convex, we have $\nabla_r \E_p L(r,Y) = 0$ implies optimality of $r$, which in turn gives $\Gamma(p)=r$.
  This shows $\Cstrict\subseteq\I$, which completes the chain of inclusions.
  As $\Gamma_\id\in\Clin$, and every property is a link of $\Gamma_\id$, the corresponding complexities are all well-defined, and the inequalities follow immediately from the inclusions.
\end{proof}

\section{Omitted Material from Section~\ref{sec:elic-complex-exampl-appl}}
\label{sec:omitted-appl}

\subsection{Proof of Theorem~\ref{thm:lin-comb-bayes-risk}}
\label{sec:proof-lin-comb-bayes-risks}

We state and prove a stronger result, which provides a form for the loss function.
The scope of loss functions given below matches those found by Fissler and Ziegel~\citeyearpar{fissler2019elicitability}; see \S~\ref{sec:elic-complex-loss-expect-shortf}.
The proof (\S~\ref{sec:proof-lin-comb-bayes-risks}) is a straightforward adaptation of Theorem~\ref{thm:main-theorem}, with the addition of the $\min(0,\alpha_i)$ term to ensure the coefficient of $L_i$ is always positive.

\begin{theorem}\label{thm:lin-comb-bayes-risk-full}
For each $i\in\{1,\ldots,m\}$ let $L_i:\reals^{k_i}\times\Y \to \reals$ be a loss eliciting $\Gamma_i:\P\to\reals^{k_i}$, with Bayes risk $\lbar_i$.
Let $\gamma(p) = \sum_{i=1}^m \alpha_i \lbar_i(p)$ for $\alpha_i \in \reals\setminus\{0\}$.
Then the loss
  \begin{equation*}
    L^*((r,a_1,\ldots,a_m),y) = \sum_{i=1}^m L'_i(a_i,y) + \sum_{i=1}^m (h(r)\alpha_i - c\min(0,\alpha_i)) L_i(a_i,y) + H(r) - h(r) r
  \end{equation*}
  elicits $\{\gamma,\Gamma_1,\ldots,\Gamma_m\}$, where $c>0$, $h:\reals\to(0,c)$ is strictly decreasing, $H(r) = \int_0^r h(x) dx$, and for each $i$, $L'_i$ is any loss weakly eliciting $\Gamma_i$.
  In particular, if $\{\gamma,\Gamma_1,\ldots,\Gamma_m\}\in\C$, $\elic_\C(\gamma) \leq \sum_{i=1}^m k_i+1$.
\end{theorem}

\begin{proof}
  Let us first unpack the coefficient $c_i$ of $L_i(a_i,y)$, which is given by
  \begin{equation*}
    c_i := h(r)\alpha_i - c \min(0,\alpha_i) =
    \begin{cases}
      h(r) \alpha_i & \alpha_i \geq 0\\
      (h(r) - c)\alpha_i & \alpha_i < 0
    \end{cases}~.
  \end{equation*}
  As we have $h:\reals \to (0,c)$, we see that $c_i > 0$ in both cases.
  For each $i$, the terms involving $a_i$ are $L'_i(a_i,y) + c_i L_i(a_i,y)$, which therefore constitute a loss function eliciting $\Gamma_i$.
  Thus, for each fixed value of $r$, the expected loss $\E_p L^*((r,a_1,\ldots,a_m),Y)$ is uniquely minimized by taking $a_i = \Gamma_i(p)$ for all $i$.
  The remainder of the proof, that the minimizing value of $r$ is $\gamma(p)$, follows directly from the proof of Theorem~\ref{thm:elic-minimum}.
\end{proof}

\subsection{Complexity of Spectral Risk Measures}
\label{sec:compl-expect-shortf}

Let $\P_s$ be any family of distributions with finite expectations such that for all $a\in\reals$ there is some $p\in\P_s$ with support contained in $[a,\infty)$.
Pareto distributions are an example of such a family.
Let $\P$ contain all mixtures of distributions in $\P_s$.
We will show that for any $\alpha_1 < \cdots < \alpha_k$, there are two distributions $p,p'$ with $q_{\alpha_i}(p) = q_{\alpha_i}(p')$ but $\rho_\mu(p) \neq \rho_\mu(p')$.
The intuition is simple: modify the distribution $p$ beyond its last quantile $q_{\alpha_k}(p)$ by moving mass toward increasing values, thus keeping the quantiles the same but increasing the expected value of the tail.

Let $p_1$ be any mixture of distributions from $\P_s$, let $\alpha_{k+1}$ such that $\alpha_k < \alpha_{k+1} < 1$, and take $a > q_{\alpha_{k+1}}(p_1)$.
Let $p_2$ be any distribution in $\P_s$ with support on $[a,\infty)$, and take $p = (\alpha_k/\alpha_{k+1}) p_1 + (1 - \alpha_k/\alpha_{k+1}) p_2$.
By construction, we have $q_{\alpha_k}(p) = q_{\alpha_{k+1}}(p_1) < a$.

To construct $p'$ we will simply replace $p_2$ with a distribution of higher mean, which will not modify the relevant quantiles.
To this end, let $a' = 1 + \E_{p_2}[Y]$, let $p_2'\in\P_s$ with support on $[a',\infty)$, and take $p' = (\alpha_k/\alpha_{k+1}) p_1 + (1 - \alpha_k/\alpha_{k+1}) p_2'$.
By the same logic as above, we have $q_{\alpha_k}(p) = q_{\alpha_{k+1}}(p_1)$, which implies $q_{\alpha_i}(p) = q_{\alpha_i}(p')$ for all $i$, as the distributions only differ in the interval $[a,\infty)$ and $a > q_{\alpha_k}(p) = q_{\alpha_k}(p')$.
Note, however, that we do have $\E_{p_2'}[Y] > a' = \E_{p_2}[Y]$.

Using the interpretation of $\ES_\alpha$ as the expected value of $Y$ conditioned on being beyond the $\alpha$ quantile, we have,
\begin{align*}
  \ES_{\alpha_i}(p)
  &= (\alpha_k/\alpha_{k+1}) \ES_{\alpha_i}(p_1) + (1 - \alpha_k/\alpha_{k+1})\E_{p_2}[Y]
  \\
  &< (\alpha_k/\alpha_{k+1}) \ES_{\alpha_i}(p_1) + (1 - \alpha_k/\alpha_{k+1})\E_{p_2'}[Y]
  \\
  &= \ES_{\alpha_i}(p')~.
\end{align*}
As the construction above works for any vector of quantiles we choose, doing it for $k+1$ sets of coefficients $\beta_i$ for which $\alpha_i$ is in the interior, gives Condition~\ref{cond:v-interior} and thus Corollary~\ref{cor:spectral-risks}.

\subsection{Losses for Expected Shortfall and Range Value at Risk}
\label{sec:elic-complex-loss-expect-shortf}

Corollary~\ref{cor:elic-complex-bayes-risk} gives us a large family of losses eliciting $\{\ES_\alpha,q_\alpha\}$.
Letting $L_\alpha(a,y) = \frac 1 \alpha(a-y)\ones_{a\geq y}-a$, we have $\ES_\alpha(p) = \inf_{a\in\reals} L_\alpha(a,p)$.
Thus we may take
\begin{equation}
  \label{eq:prop-notes-3}
  L((r,a),y) = L'(a,y) + H(r) + h(r)(L_\alpha(a,y) - r)~,
\end{equation}
where $h(r)$ is positive and strictly decreasing, $H(r) = \int_0^r h(x)dx$, and $L'(a,y)$ is any other loss eliciting $q_\alpha$, the full characterization of which is given in \citet[Theorem 9]{gneiting2011making}:
\begin{equation}
  \label{eq:prop-notes-4}
  L'(a,y) = (\ones_{a\geq y}-\alpha)(f(a)-f(y)) + g(y)~,
\end{equation}
where is $f:\reals\to\reals$ is nondecreasing and $g$ is an arbitrary $\P$-integrable function.
As an aside, Gneiting~\citeyearpar{gneiting2011making} assumes $L(x,y)\geq 0$, $L(x,x)=0$, $L$ is continuous in $x$, $dL/dx$ exists and is continuous in $x$ when $y\neq x$; we add $g$ because we do not normalize.
Hence, losses of the following form suffice:
\begin{align*}
  L((r,a),y)
  & = (\ones_{a\geq y}-\alpha)(f(a)-f(y))\\&~~~~+ \frac 1 \alpha h(r) \ones_{a\geq y} (a-y)  - h(r) (a + r) + H(r) + g(y)~.
\end{align*}
Comparing our family of losses $L((r,a),y)$ to the characterization given by \citet[Cor. 5.5]{fissler2016higher}, we see that we recover
 all possible scores for this case, at least when restricting to the assumptions stated in their Theorem 5.2(iii).  Note however that due to a differing convention in the sign of $\ES_\alpha$, their loss is given by $L((-x_1,x_2),y)$.

Similarly, the losses we obtain for $\mathrm{RVaR}_{\alpha,\beta}$ from Theorem~\ref{thm:lin-comb-bayes-risk} are given by the following, where $f_1,f_2$ are nondecreasing and again $g$ is an arbitrary $\P$-integrable function.
\begin{align*}
  L^*((r,a_1,a_2),y)
  & = (\ones_{a_1\geq y}-\alpha)(f_1(a_1)-f_1(y)) + (\ones_{a_2\geq y}-\beta)(f_2(a_2)-f_2(y))\\
  &~~~~- (h(r)-c) \frac \alpha {\beta-\alpha} L_\alpha(a_1,y)
  + h(r) \frac \beta {\beta-\alpha} L_\beta(a_2,y)
    + H(r) - h(r) r + g(y)\\
  & = (\ones_{a_1\geq y}-\alpha)(f_1(a_1)-f_1(y)) + (\ones_{a_2\geq y}-\beta)(f_2(a_2)-f_2(y)) + c\frac{\alpha}{\beta-\alpha} L_\alpha(a_1,y)\\
  &~~~~+ h(r) \left(\frac 1 {\beta-\alpha} \left(\beta L_\beta(a_2,y) - \alpha L_\alpha(a_1,y)\right) - r\right)
    + H(r) + g(y)\\
  & = (\ones_{a_1\geq y}-\alpha)(f_1'(a_1)-f_1'(y)) + (\ones_{a_2\geq y}-\beta)(f_2(a_2)-f_2(y))\\
  &~~~~+ h(r) \left(\frac 1 {\beta-\alpha} \left(\beta L_\beta(a_2,y) - \alpha L_\alpha(a_1,y)\right) - r\right)
    + H(r) + g'(y)~,
\end{align*}
where $f_1'(a_1) = f_1(a_1) + \tfrac{c}{\beta-\alpha}a_1$ and $g'(y) = g(y) + \frac{c\alpha}{\beta-\alpha}y$.
Comparing now with Fissler and Ziegel~\citeyearpar{fissler2019elicitability}, modulo the difference in sign convention noted above, we see that this family of losses is equivalent to \citet[eq.~(3.2)]{fissler2019elicitability}, as the condition that $a_1 \mapsto f_1'(a_1) - a_1 h(r)/(\beta-\alpha)$ be strictly increasing is equivalent to $f_1(a_1) = f_1'(a_1) - a_1 c/(\beta-\alpha)$ being nondecreasing.
Recall that $h:\reals\to(0,c)$; without loss of generality we may assume $c$ is the supremum.

\subsection{Complexity of Variantiles}
\label{sec:omitted-variantiles}

To establish Corollaries~\ref{cor:expectile} and \ref{cor:multivariate-expectile}, we will show three statements:
(1) $\{\Var^{(k)}_\tau,\mu^{(k)}_\tau\} \in \I$ and furthermore $\{\Var^{(k)}_\tau,\mu^{(k)}_\tau\} \in \Cstrong$ when $\Y$ is bounded,
(2) Condition~\ref{cond:v-interior} holds for $\mu^{(k)}_\tau$ and some $r\in\reals^k$, and
(3) there are at least two distributions $p,p'\in\P$ with $\mu^{(k)}_\tau(p)=\mu^{(k)}_\tau(p')=r$ but $\Var^{(k)}_\tau(p)\neq\Var^{(k)}_\tau(p')$.
By assumption, Statements 2 and 3 only require proof for the specific case of $\P=\gaussians$.
Both corollaries will then follow from Proposition~\ref{prop:tighter-lower-bound}.

\paragraph{Statement 1}
Recall that we define $L(x,y) = \|y-x\|_2(\|y-x\|_2+\inprod{\tau,y-x})$.
\citet[Theorems 4.1, 4.3]{herrmann2018multivariate} show that $L$ is differentiable and strictly convex, from which we conclude that $V(x,y) = \nabla_x L(x,y)$ is an identification function for $\mu^{(k)}_\tau$; see the proof of Proposition~\ref{prop:nested-convex-classes-long}.
Thus, by the proof of Corollary~\ref{cor:bayes-risk-ident-lower-bound}, we have $\{\mu^{(k)}_\tau,\Var^{(k)}_\tau\} \in \I$.

To show strong convexity, let $\Lambda_\tau(v) = \|v\|_2(\|v\|_2+\inprod{\tau,v})$, so that we have $L(x,y) = \Lambda_\tau(y-x) = \|y-x\|_2(\|y-x\|_2+\inprod{\tau,y-x})$; we will show that $\Lambda_\tau$ is strongly convex.
In what follows, we drop the subscript in the norm and write $\|\cdot\|=\|\cdot\|_2$.
The proof given in \citet[Theorem 4.3]{herrmann2018multivariate} that $\Lambda_\tau$ is strictly convex proceeds by showing $D(v,w) = \tfrac 1 2 \Lambda_\tau(v) + \tfrac 1 2 \Lambda_\tau(w) - \Lambda_\tau(\tfrac 1 2 v+ \tfrac 1 2 w)$ is strictly positive.
This is done by expanding $4\cdot D$,
\begin{equation}
  \label{eq:mult-expectile-D}
  4 D(v,w) = \|v-w\|^2 + 2\|v\|\inprod{\tau,v} + 2\|w\|\inprod{\tau,w} - \|v+w\|\inprod{\tau,v+w}~,
\end{equation}
and showing that $D(v,w) \geq 0$ whenever $\|\tau\| \leq 1$, with an inequality for $v\neq w$ if $\|\tau\|<1$ \citep[Theorem 4.2]{herrmann2018multivariate}.
Convexity follows as $\Lambda_\tau$ is continuous.

By standard results \citep[Proposition B.1.1.2]{urruty2001fundamentals}, strong convexity of $\Lambda_\tau$ would follow by showing $D(v,w) \geq c\|v-w\|^2$ for some $c$.
Examining eq.~\eqref{eq:mult-expectile-D}, we see that all terms apart from the $\|v-w\|^2$ term are linear in $\tau$.
Thus, replacing $\tau$ by $\tau/\|\tau\|$ in eq.~\eqref{eq:mult-expectile-D} still satisfies \citet[Theorem 4.2]{herrmann2018multivariate}, giving us
\begin{align*}
  0 \leq& \|v-w\|^2 + 2\|v\|\inprod{\tau/\|\tau\|,v} + 2\|w\|\inprod{\tau/\|\tau\|,w} - \|v+w\|\inprod{\tau/\|\tau\|,v+w}\\
= & \tfrac 1 {\|\tau\|} \left( \|\tau\| \|v-w\|^2 + 2\|v\|\inprod{\tau,v} + 2\|w\|\inprod{\tau,w} - \|v+w\|\inprod{\tau,v+w} \right)\\
= & \tfrac 1 {\|\tau\|} \left( D(v,w) - (1-\|\tau\|) \|v-w\|^2 \right)~.
\end{align*}
Thus, letting $c = 1-\|\tau\| > 0$, we have strong convexity of $\Lambda_\tau$.
We conclude that $L$ is strongly convex.
When $\Y\subseteq\reals$ is bounded, Proposition~\ref{prop:str-cvx-elic} gives $\{\mu^{(k)}_\tau,\Var^{(k)}_\tau\} \in \Cstrong$, as desired.

\paragraph{Statement 2}
\newcommand{\ksphere}{{S^{k-1}}}
We will prove the multivariate case, which subsumes the univariate case.
The proof will make use of support functions and the Hausdorff metric; we now recall the necessary definitions and standard results.
The support function $h_K:\ksphere\to\reals\cup\{\infty\}$ of a set $K \subseteq \reals^k$ is given by $h_K(v) = \sup_{x\in K} \inprod{v,x}$.
The Hausdorff distance $d_H(A,B)$ between two closed sets $A,B\subseteq \reals^k$, is defined by $d_H(A,B) = \max\{\sup_{x\in A} d(x,B), \sup_{x\in B} d(x,A)\}$, where $d(x,S) = \min_{y\in S} \|y-x\|$ is the distance between a point $x\in\reals^k$ and a set $S\subseteq \reals^k$.
We have the following facts.
\begin{enumerate}
\item For all compact convex $A,B\subseteq\reals^k$, $\max_{v\in\ksphere}|h_A(v)-h_B(v)| = d_H(A,B)$.
\item For all compact $A,B\subseteq\reals^k$, $d_H(\conv A,\conv B) \leq d_H(A,B)$.
\item For all convex $A\subseteq\reals^k$, we have $0\in\interior A \iff \forall v\in\ksphere\; d_A(v) > 0$.
\end{enumerate}
The first and third fact may be found in \citet[Theorem C.3.3.6 \& Theorem C.2.2.3(iii)]{urruty2001fundamentals}.
The second follows by taking a convex combination $x = \sum_i \lambda_i a_i$ of elements $a_i\in A$, and approximating each $a_i$ within $d_H(A,B)$ by elements $b_i \in B$.

To show Statement 2, we must establish $0 \in \interior\{\E_p[V(x,Y)]:p\in\gaussians\} \subseteq \reals^k$ for some identification function $V$ and some $x\in\reals^k$.
We will take $V$ to be the identification function from Statement 1, and $x=0$.
For any $p\in\gaussians$, we have
\begin{align}
  \E_p[V(0,Y)]
  \label{eq:expectile-iden-func}
  &= \E_p\left[2Y + \frac Y {\|Y\|}\inprod{\tau,Y} + \|Y\|\tau\right]~.
\end{align}
Let $f: p \mapsto \E_pV(0,Y)$.
It therefore suffices to show $0\in\interior \phi(\gaussians)$.

Letting $\ksphere = \{x\in\reals^k : \|x\|=1\}$ be the unit sphere, define $z:\ksphere \to \reals^k$ by $z : \mu \mapsto 2\mu + \inprod{\tau,\mu}\mu + \tau$, which is the value of $\phi(p)$ when $p$ is sufficiently concentrated around $\mu$.
Let $Z = z(\ksphere) = \{z(x) : x\in \ksphere\}$ and $C=\conv Z$.
For all $v\in\ksphere$, we have
\begin{align*}
  h_C(v)
  &\geq \inprod{v,z(v)} \\
  &= 2\inprod{v,v} + \inprod{\tau,v} \inprod{v,v} + \inprod{v,\tau} \\
  &= 2 + 2 \inprod{\tau,v} \\
  &\geq 2(1-\|\tau\|) > 0~.
\end{align*}
Let $\epsilon = 1-\|\tau\| > 0$.
$Z$ is compact, as the continuous image of a compact set, and thus there exists a finite subset $Z'\subseteq Z$ with $d_H(Z',Z) < \epsilon$.
By definition of $Z$, we can write $Z' = z(S')$ for some finite $S' \subseteq \ksphere$.

From Lemma~\ref{lem:gaussian-sigma-limit}, for all $\mu\in S'$ we have some $\sigma(\mu) > 0$ such that $\|\phi({\cal N}(\mu,\sigma(\mu) I)) - z(\mu)\| < \epsilon$, where $\cal N$ is the multivariate Gaussian distribution.
Now take $\sigma = \min_{\mu \in S'} \sigma(\mu)$ and define $P=\{{\cal N}(\mu,\sigma I) : \mu\in S'\}$ and $Z'' = \phi(P)$.
We therefore have $d_H(Z',Z'') < \epsilon$.
Letting $C'' = \conv Z''$, we have
\[d_H(C'',C) \leq d_H(Z'',Z) \leq d_H(Z'',Z') + d_H(Z',Z) < 2\epsilon~.\]
Thus, for all $v\in\ksphere$, we have
\begin{align*}
  h_{C''}(v) > h_C(v) - 2\epsilon > 2\epsilon-2\epsilon = 0~,
\end{align*}
giving $0 \in \interior C''$.
As $\gaussians$ is convex, $f$ is linear, and $P \subseteq \gaussians$, we have $C'' = \conv \phi(P) = \phi(\conv P) \subseteq \phi(\gaussians)$.
We conclude $0\in\interior \phi(\gaussians)$, as desired.

\begin{lemma}\label{lem:gaussian-sigma-limit}
  For all $\mu \in \ksphere$,
  we have $\lim_{\sigma \to 0^+} \E_{{\cal N}(\mu,\sigma I)} [V(0,Y)] = z(\mu)$.
\end{lemma}
\begin{proof}
  Expanding eq.~\eqref{eq:expectile-iden-func}, we have
  \begin{align}
    \E_p[V(0,Y)] &= 2\E_p[Y] + \E_p\left[\inprod{\tau,Y}\frac Y {\|Y\|}\right] + \E_p[\|Y\|]\tau~.
  \end{align}
  Fix $\mu\in\ksphere$ and let $\{\sigma_n\}_{n\in\N}$ be a positive real sequence converging to zero.
  Define $Y^{(n)}\sim{\cal N}(\mu,\sigma_n^2 I)$.
  We thus have $Y^{(n)} \to \mu$ in probability.
Fix a coordinate $j\in\{1,\ldots,k\}$, let $f_j(y) = \inprod{\tau,y} \tfrac {y_j} {\|y\|} $ and $g(y) = \|y\|$, and define $X^{(n)} = f_j(Y^{(n)})$ and $Z^{(n)} = g(Y^{(n)})$.
As $f_j$ and $g$ are both continuous functions to the reals, we thus have both $X^{(n)} \to f_j(\mu)$ and $Z^{(n)} \to g(\mu) = 1$ in probability.
Uniform integrability of $X^{(n)}$ follows from the observation that $|X^{(n)}| \leq |Y^{(n)}_j|\|\tau\| < |Y^{(n)}_j|$, and appealing to uniform integrability of $Y^{(n)}_j\sim{\cal N}(\mu_j,\sigma_n^2)$.
For uniform integrability of $Z^{(n)}$, observe that $\|Y^{(n)}\|$ has a noncentral $\chi^2$ distribution, and thus $\E|Z^{(n)}|^2 = \E\|Y^{(n)}\|^2 = k\sigma_n^2 + \|\mu\| \leq k\sigma_1^2 + 1$; uniform integrability now follows from \citet[eq. (25.13)]{billingsley2008probability}.
We therefore have $\E[\inprod{\tau,Y^{(n)}}\frac {Y^{(n)}}{\|Y^{(n)}\|}] = \E[X^{(n)}] \to f_j(\mu) = \inprod{\tau,\mu}\mu_j$ and $\E[\|Y^{(n)}\|] = \E[Z^{(n)}] \to g(\mu) = 1$ \citep[Theorem 25.12]{billingsley2008probability}.
We conclude $\lim_{\sigma \to 0^+} \E_{{\cal N}(\mu,\sigma I)} [V(0,Y)]_j = 2\mu_j + f_j(\mu) + g(\mu) \tau_j = z(\mu)_j$.
\end{proof}

\paragraph{Statement 3}
We first illustrate the univariate case for intuition.
Let $p\in\gaussians$ with $\mu_\tau(p) = 0$ and $\Var_\tau(p) > 0$.
The latter is implied by nonzero variance.
Letting $X \sim p$, and $\lambda>0$, we have
\begin{equation*}
  \E\left[|\ones_{0\geq \lambda X}-\tau|(0-\lambda X)\right] =
  \lambda \E\left[|\ones_{0\geq X}-\tau|(0-X)\right] = 0~, 
\end{equation*}
meaning $\mu_\tau(p_\lambda)=0$ as well, where $p_\lambda$ is the law of $\lambda X$; note
$p_\lambda\in\gaussians$.
The variantile changes, however, whenever $\lambda \neq 1$:
\begin{equation*}
  \Var_\tau(p_\lambda) = \E\left[|\ones_{0\geq \lambda X}-\tau|(\lambda X)^2\right]=
  \lambda^2 \E\left[|\ones_{0\geq X}-\tau|X^2\right]=
  \lambda^2 \Var_\tau(p)~.
\end{equation*}
The statement now follows.

The multivariate case follows similarly.
Again let $p\in\gaussians$ satisfy $\mu^{(k)}_\tau(p)=0$ with a positive-definite covariance matrix, thus implying $\Var^{(k)}_\tau(p) > 0$.
Let $X\sim p$ and $\lambda>0$.
Let $p_\lambda$ be the law of $\lambda X$, and note $p_\lambda \in \gaussians$.
We now have
\begin{align*}
  \mu^{(k)}_\tau(p_\lambda)
  &= \argmin_{x\in\reals^k} \;\E\left[\|\lambda X-x\|_2(\|\lambda X-x\|_2+\inprod{\tau,\lambda X-x})\right] \\
  &= \argmin_{x\in\reals^k} \;\E\left[\|X-x/\lambda\|_2(\|X-x/\lambda\|_2+\inprod{\tau,X-x/\lambda})\right] \\
  &= \argmin_{x\in\reals^k} \;\E\left[\|X-x\|_2(\|X-x\|_2+\inprod{\tau,X-x})\right]
  \\
  &= \mu^{(k)}_\tau(p) = 0~.
\end{align*}
Turning to the variantile, we similarly have
\begin{align*}
  \Var^{(k)}_\tau(p_\lambda)
  &= \min_{x\in\reals^k} \;\E\left[\|\lambda X-x\|_2(\|\lambda X-x\|_2+\inprod{\tau,\lambda X-x})\right]
  \\
  &= \min_{x\in\reals^k} \lambda^2\E[\|X-x/\lambda\|_2(\|X-x/\lambda\|_2+\inprod{\tau,X-x/\lambda})]
  \\
  &= \lambda^2  \Var^{(k)}_\tau(p)~,
\end{align*}
which again gives a different value whenever $\lambda\neq 1$.

\section{Omitted Material from Section~\ref{sec:basic-results}}

\subsection{Identification Lower Bounds}

\begin{proof}[of Lemma~\ref{lem:main-iden-lower-bound}]
  We will simply apply Lemma~\ref{lem:lin-alg-nested-level-sets} with $\V = \spn \P$, $C=\P$, and $S = \Gamma_r$.
  Let $f : \spn \P \to \reals^k$ given by $f(q) = V(r,q)$, where we interpret $q$ as a signed measure.
  By Condition~\ref{cond:v-interior}, we have $0\in\interior f(\P)$.
  Now consider some $\hat V:\hat\R\times\Y\to\reals^\ell$ identifying $\hat\Gamma$, where $\hat\R = \hat\Gamma(\P)$ and $\ell \in \N$, such that $\hat\Gamma$ refines $\Gamma$.
  Refinement implies that for any $p\in \Gamma_r$, there is some $\hat r\in\hat\R$ such that $p \in \hat\Gamma_{\hat r} \subseteq \Gamma_r$.
  For any such $\hat r$, we may define $\hat f: \spn P \to \reals^\ell$ by $\hat f(q) = \hat V(\hat r,q)$.
  For any $p\in\Gamma_r$, we therefore have a linear $\hat f : \spn \P \to \reals^\ell$ such that $p \in \P \cap \ker \hat f \subseteq \Gamma_r$.
  The conditions of Lemma~\ref{lem:lin-alg-nested-level-sets} are now satisfied, giving $\ell \geq k$, and thus $\iden(\Gamma) \geq k$.  
\end{proof}

\begin{lemma}\label{lem:lin-alg-nested-level-sets}
  Let $\V$ be a real vector space.
  Let $f:\V\to\reals^k$ be linear and $C\subseteq \V$ convex with $\spn C = \V$.
  Suppose $0 \in \interior f(C)$.
  Let $S = C \cap \ker f$.
  If
  $\ell\in\N$ is such that
  for all $v\in S$, there exists a linear $\hat f:\V\to\reals^\ell$ with $v \in C \cap \ker \hat f \subseteq S$, then $\ell \geq k$.
\end{lemma}
\begin{proof}
  The condition $0 \in \interior f(C)$ is equivalent to the existence of some $v_1,\ldots v_{k+1} \in C$ such that $0\in\interior\conv\{f(v_i) : i\in\{1,\ldots,k+1\}\}$.
  Let $\alpha_1,\ldots,\alpha_{k+1}>0$, $\sum_{i=1}^{k+1} \alpha_i = 1$, such that $\sum_{i=1}^{k+1} \alpha_i f(v_i) = 0$.
  As these are barycentric coordinates,
  this choice of $\alpha_i$ is unique, a fact which will be important later.
  We will take $v = \sum_{i=1}^{k+1} \alpha_i v_i$, an element of $C$ by convexity, and thus an element of $S$ as $f(v)=0$.

  Let $\hat f:\V\to\reals^\ell$ be linear with $v \in \hat S := C \cap \ker \hat f \subseteq S$.
  Let $\beta_1,\ldots,\beta_{k+1}\in\reals$, $\sum_{i=1}^{k+1} \beta_i = 0$, such that $\sum_{i=1}^{k+1} \beta_i \hat f(v_i) = 0$.
  We will show that the $\beta_i$ must be identically zero, i.e. that $\{\hat f(v_i):i\in\{1,\ldots,k+1\}\}$ are affinely independent.
  By construction, $v' := \sum_{i=1}^{k+1} \beta_iv_i \in \ker \hat f$, and as $v\in\ker\hat f$, for all $\lambda > 0$ we have $v_\lambda := v + \lambda v' = \sum_{i=1}^{k+1} (\alpha_i + \lambda \beta_i) v_i \in \ker \hat f$.
  Taking $\lambda$ sufficiently small, we have $\gamma_i := \alpha_i + \lambda \beta_i > 0$ for all $i$, and $\sum_{i=1}^{k+1} \gamma_i = \sum_{i=1}^{k+1} \alpha_i + \lambda\sum_{i=1}^{k+1} \beta_i = 1$.
  By convexity of $C$, we have $v_\lambda \in C$.
  Now $v_\lambda \in C \cap \ker \hat f \subseteq S = C \cap \ker f$, and in particular $v_\lambda \in \ker f$.
  Thus, $f(v_\lambda) = \sum_{i=1}^{k+1} \gamma_i f(v_i) = 0$.
  By the uniqueness of barycentric coordinates, for all $i\in\{1,\ldots,k+1\}$, we must have $\gamma_i = \alpha_i$ and thus $\beta_i = 0$, as desired.

  As $\hat f(C)$ contains $k+1$ affinely independent points, we have $\ell \geq \dim \im \hat f \geq k$, completing the proof.

  We make one final observation for the case $\ell=k$.
  By affine independence, the set $\conv\{\hat f(v_i) : i\in\{1,\ldots,k+1\}\}$ has dimension $k$ in $\reals^k$.
  As $0 = \hat f(v) = \sum_{i=1}^{k+1} \alpha_i \hat f(v_i)$, and $\alpha_i > 0$ for all $i$, we conclude $0\in\interior\conv\{\hat f(v_i) : i\in\{1,\ldots,k+1\}\} \subseteq \interior \hat f(C)$.
\end{proof}

\subsection{Expectations and Quantiles}
\label{sec:expect-quant}

\begin{proof}[of Lemma~\ref{lem:elic-complex-means}]
  Let $\ell = \affdim(\Gamma(\P))$, and let $r_0 \in \relint\, \Gamma(\P)$.
  Then $\V = \spn \{ \Gamma(p) - r_0 : p\in\P \} \subseteq \reals^k$ is a vector space of dimension $\ell$.
  Let $M = [v_1 \ldots v_\ell] \in \reals^{k\times \ell}$ where $v_1,\ldots,v_\ell\in\reals^k$ is a basis of $\V$.
  Now define $V:\Gamma(\P)\times\Y\to\reals^\ell$ by $V(r,y) = M^+(\phi(y) - r)$, where $M^+$ is the Moore--Penrose pseudoinverse of $M$.
  Clearly $\E_p[\phi(Y)]=r \implies V(r,p)=0$, and as $\E_p[\phi(Y)]-r\in\im M$, we have $M^+(\E_p[\phi(Y)]-r)=0 \implies \E_p[\phi(Y)]-r = 0$ by properties of the pseudoinverse $M^+$.
  Thus $\iden(\Gamma)\leq \ell$.
  Moreover, as $r_0\in\relint\,\Gamma(\P)$, we have $M^+ r_0 \in \int \{M^+r : r\in\Gamma(\P)\}$, and thus $0 \in \int \{V(r,p) : r\in\Gamma(\P)\}$, satisfying Condition~\ref{cond:v-interior}.
  Lemma~\ref{lem:main-iden-lower-bound} now gives $\iden(\Gamma) = \ell$.
  Elicitability follows by letting $\Gamma'(p) = M^+ (\E_p[\phi(Y)]-r_0) = \E_p[M^+(\phi(Y)-r_0)] \in \reals^\ell$ with link $f(r') = M r' + r_0$; $\Gamma'$ is of course elicitable as a linear property.
\end{proof}

\begin{proof}[of Lemma~\ref{lem:elic-complex-quantiles}]
  The function $V(r,y)_i = \ones\{y\leq r_i\} - \alpha_i$ identifies $\Gamma$, as we have $\E_F V(r,Y) = 0 \iff \forall i\; F(r_i)=\alpha_i \iff \forall i\; r_i = q_{\alpha_i}(F)$.
  Thus, as quantiles are elicitable, $\elici(\Gamma) \leq k$.
  As Condition~\ref{cond:quantiles} implies Condition~\ref{cond:v-interior} for this $V$, the lower bound follows immediately from Lemma~\ref{lem:main-iden-lower-bound}.
\end{proof}

\subsection{Convex Functions of Means, for Proposition~\ref{prop:elic-complex-classes}}
\label{sec:conv-funct-means}

Consider a property of the form $\gamma(p) = G(\E_p[\phi(Y)])$ for some strictly convex function $G:\reals^k\to\reals$ and $\P$-integrable $\phi:\Y\to\reals^k$.
To avoid degeneracies, we assume the set $\{\E_p[\phi(Y)]:p\in\P\}$ has affine dimension $k$, which from Lemma~\ref{lem:elic-complex-means} ensures that the property $\Gamma:p\mapsto\E_p[\phi(Y)]$ has $\elic_\C(\Gamma)=k$ for all $\C$ satsifying $\Clin\subseteq\C\subseteq\I$.
Letting $\{dG_r\}_{r\in\reals^k}$ be a selection of subgradients of $G$, the loss $L(r,y) = -(G(r) + dG_r\cdot (\phi(y)-r))$ elicits~$\Gamma$, and moreover we have $\gamma(p)=-\lbar(p)$; see e.g.\ \citet{frongillo2015vector}.
One easily checks that $\lbar = (-G) \circ \Gamma$.
Theorem~\ref{thm:bayes-risk-lower-bound} now immediately gives $\elic_\C(\lbar) = \elic_\C(\Gamma) = k$ for all $\Clin\subseteq\C\subseteq\I$.
We summarize this discussion as follows.

\begin{corollary}
  \label{cor:func-mean}
  Let $\phi:\Y\to\reals^k$, $k\in\N$, be $\P$-integrable with $\affdim\{\E_p[\phi(Y)]:p\in\P\}=k$.
  Then for any strictly convex $G:\reals^k\to\reals$, the property $\gamma : p \mapsto G(\E_p[\phi(Y)])$ has $\elic_\C(\gamma)=k$ for all $\C$ satisfying $\Clin\subseteq\C\subseteq\I$.
\end{corollary}

\section{Omitted Material from Section~\ref{sec:elic-complex-bayes-risk}}

\subsection{Proof of Lemma~\ref{lem:elic-complex-bayes-concave}}
\label{sec:proof-lemma-bayes-concave}

\begin{lemma}[\citet{frongillo2014general}]
  \label{lem:elic-complex-convex-flat}
  Let $G:X\to\reals$ convex for some convex subset $X$ of a vector space $\V$, and let $d\in\partial G_x$ be a subgradient of $G$ at $x$.  Then for all $x'\in X$ we have
  \[ d\in\partial G_{x'} \iff G(x)-G(x') = d(x-x')~. \]
\end{lemma}

\begin{lemma}
  \label{lem:elic-complex-convex-kink}
  Let $G:X\to\reals$ convex for some convex subset $X$ of a vector space $\V$.
  Let $x,x'\in X$ and $x_\lambda = \lambda x + (1-\lambda) x'$ for some $\lambda \in (0,1)$.
  If there exists some $d\in\partial G_{x_\lambda} \setminus (\partial G_x\cup \partial G_{x'})$, then $G(x_\lambda) < \lambda G(x) + (1-\lambda) G(x')$.
\end{lemma}
\begin{proof}
  By the subgradient inequality for $d$ at $x_\lambda$ we have $G(x)-G(x_\lambda) \ge d(x-x_\lambda)$, and furthermore Lemma~\ref{lem:elic-complex-convex-flat} gives us $G(x)-G(x_\lambda) > d(x-x_\lambda)$ since otherwise we would have $d\in\partial G_x$.  Similarly for $x'$, we have $G(x')-G(x_\lambda) > d(x'-x_\lambda)$.  

Adding $\lambda$ of the first inequality to $(1-\lambda)$ of the second gives
\begin{align*}
  \lambda G(x) + (1-\lambda) G(x') - G(x_\lambda) &> \lambda d(x-x_\lambda) + (1-\lambda) d(x'-x_\lambda) \\
 &= \lambda (1-\lambda) d(x-x') + (1-\lambda) \lambda d(x'-x) 
= 0~,
\end{align*}
where we used linearity of $d$ and the identity $x_\lambda = x' + \lambda (x-x')$.
\end{proof}

Lemma~\ref{lem:elic-complex-bayes-concave} follows from the following result.

\begin{lemma}
  \label{lem:elic-complex-bayes-concave-multivalued}
  Suppose loss $L$ with Bayes risk $\lbar$ elicits $\Gamma:\P\to 2^\R$.  Then for any $p,p'\in\P$ with $\Gamma(p)\cap\Gamma(p')=\emptyset$, we have $\lbar(\lambda p + (1-\lambda) p') > \lambda \lbar(p) + (1-\lambda) \lbar(p')$ for all $\lambda\in(0,1)$.
\end{lemma}
\begin{proof}
  Let $G = -\lbar$, which is the expected score function for the positively-oriented scoring rule $S = -L$.
  By Theorem 2 of \citet{frongillo2014general}, we have some subset $\D\subseteq \partial G$ of subgradients of $G$, and bijection $\varphi:\Gamma(\P)\to\D$, such that $\Gamma(p) = \varphi^{-1}(\D\cap\partial G_p)$.
  In other words, $\Gamma$ is a relabeling of a selection of subgradients of $G$: there is a subgradient $d_r = \varphi(r)$ associated to each report value $r\in\Gamma(\P)$, and $d_r \in \partial G_p \iff r \in \Gamma(p)$.

  Observe that for any distributions $q,q'\in\P$, if $\Gamma(q)\cap\Gamma(q')=\emptyset$, then for any $r\in\Gamma(q)$ and $d_r = \varphi(r)$, we have $d_r \in \partial G_q \setminus \partial G_{q'}$.
  Otherwise, since $d_r \in \D\cap\partial G_q$ by definition, we would have $d_r \in \D\cap\partial G_{q'}$ as well, and thus $r = \varphi^{-1}(d_r) \in \varphi^{-1}(\D\cap\partial G_{q'}) = \Gamma(q')$, a contradiction.

  Assume first that $\Gamma(p_\lambda)$, $\Gamma(p)$, and $\Gamma(p')$ are all disjoint sets.
  By the above observation, taking any $d\in\varphi(\Gamma(p_\lambda))$, we have $d\in\partial G_{p_\lambda}$ but $d\notin \partial G_p\cap\partial G_{p'}$.
  The conclusion then follows by Lemma~\ref{lem:elic-complex-convex-kink}.

  Otherwise, we have $r \in \Gamma(p_\lambda) \cap \Gamma(p)$ without loss of generality, and letting $d_r = \varphi(r)$, we have $d_r \in \partial G_{p_\lambda} \cap \partial G_p$ by definition of $\varphi$.
  Now assume for a contradiction that $G(p_\lambda) = \lambda G(p) + (1-\lambda) G(p')$.
  By Lemma~\ref{lem:elic-complex-convex-flat} for $d_r$ we have $G(p) - G(p_\lambda) = d_r(p-p_\lambda) = \tfrac{(1-\lambda)}{\lambda} d_r(p_\lambda-p')$.
  Solving for $G(p)$ and substituting into the previous equation gives $(1-\lambda)$ times the equation $G(p_\lambda) = d_r(p_\lambda-p') + G(p')$, and applying Lemma~\ref{lem:elic-complex-convex-flat} one more gives $d_r \in \partial G_{p'}$.
  We now have a contradiction to the observation above, as we have assumed $\Gamma(p)\cap\Gamma(p')=\emptyset$.
\end{proof}

Lemma~\ref{lem:elic-complex-bayes-concave} now follows immediately; given $\Gamma:\P\to\R$ from Lemma~\ref{lem:elic-complex-bayes-concave}, we simply apply Lemma~\ref{lem:elic-complex-bayes-concave-multivalued} to the property $\Gamma':\P\to 2^\R$ given by $\Gamma'(p) = \{\Gamma(p)\}$.

As remarked in \S~\ref{sec:extens-non-conv}, the restriction that $\P$ be convex is not crucial to our results.
For non-convex $\P$, one would extend the Bayes risk $\lbar$ to the convex hull $\conv\,\P$ of $\P$, by writing $\lbar(p) = \arginf_{r\in\R} L(r,p)$, where of course $\R = \Gamma(\P)$.
One can then extend $\Gamma$ by adding new reports, suggested by Theorem 2 of \citet{frongillo2014general}, so that $\Gamma$ is non-redundant and nonempty on $\conv\,\P$, but coincides with its previous definition on $\P$.
Lemma~\ref{lem:elic-complex-bayes-concave-multivalued} then follows as usual, and since $\lbar$ and $\P$ are unchanged on $\P$, the result holds that $\lbar(\lambda p + (1-\lambda) p') > \lambda \lbar(p) + (1-\lambda) \lbar(p')$ for all $\lambda\in(0,1)$ such that $\lambda p + (1-\lambda) p'\in\P$.

\subsection{Proof of Corollary \ref{cor:bayes-risk-ident-lower-bound}}

We will make use of Lemma~\ref{lem:lin-alg-nested-level-sets} together with the following.

\begin{lemma}\label{lem:lin-alg-span}
  Let $\V$ be a real vector space.
  Let $f:\V\to\reals^k$ be linear, $C\subseteq \V$ convex with $\spn C = \V$, and let $S = C \cap \ker f$.
  If $0 \in \interior f(C)$ then $\spn S = \ker f$.
\end{lemma}

\begin{proof}
  As $\ker f$ is a subspace and $S \subseteq \ker f$, we have that $\spn S$ is a subspace of $\ker f$.
  Applying the universal property of quotient spaces, we have linear maps $\pi : \V \to \V / \spn S$ and $g : \V / \spn S \to \reals^k$ such that $f = g \circ \pi$.
  By assumption, $\{0\} = \pi(S) = \pi(\ker f \cap C) = \pi(\ker f) \cap \pi(C) = \ker g \cap \pi(C)$.
  We will show the stronger statement that $\ker g = \{0\}$, which implies $\ker f = \ker \pi = \spn S$.

  As $\spn C = \V$, we have $\spn \pi(C) = \V / \spn S$.
  Thus, for any $x \in \V / \spn S$ we may write $x = \sum_{i=1}^m \alpha_i x_i$ for $\alpha_i \in \reals$ and $x_i \in \pi(C)$.
  As $0\in\interior f(C) = \interior g(\pi(C))$, there is an open ball $B$ with $0\in B \subseteq g(\pi(C))$.
  For each $i\in\{1,\ldots,m\}$, the line containing $g(x_i)$ and $0$ intersects $B$.
  In particular, for some sufficiently small $\epsilon>0$, for each $i\in\{1,\ldots,m\}$ we have some $y_i\in\pi(C)$ with $g(y_i) = -\epsilon g(x_i)$.
  By linearity, $g(\tfrac 1 {1+\epsilon} (y_i + \epsilon x_i)) = 0$, and from convexity of $\pi(C)$ we also have $\tfrac 1 {1+\epsilon} (y_i + \epsilon x_i) \in \pi(C)$.
  From the observation $\ker g \cap \pi(C) = \{0\}$ we now have $y_i = -\epsilon x_i$.

  Define $\beta_i =
  \begin{cases}
    \alpha_i & \alpha_i \geq 0 \\
    -\alpha_i/\epsilon & \alpha_i < 0
  \end{cases} \geq 0$
  and
  $w_i =
  \begin{cases}
    x_i & \alpha_i \geq 0 \\
    y_i & \alpha_i < 0
  \end{cases}$
  for $i\in\{1,\ldots,m\}$, and set $\beta_{m+1} = 1$, $w_{m+1} = 0 \in \pi(C)$.
  Let $\beta = \sum_{i=1}^{m+1} \beta_i \geq 1$.
  For all $i\in\{1,\ldots,m+1\}$, as $w_i,0 \in \pi(C)$, we have $w_i/\beta \in \pi(C)$ by convexity.
  Thus,
  \begin{align*}
    \sum_{i=1}^{m+1} (\beta_i/\beta) w_i
    &= \frac 1 \beta \left( \sum_{i\in\{1,\ldots,m\}:\alpha_i \geq 0} \alpha_i x_i + \sum_{i\in\{1,\ldots,m\}:\alpha_i < 0} -\alpha_i \frac {y_i} \epsilon + 1 \cdot 0 \right)
    \\
    &= \frac 1 \beta \sum_{i=1}^m \alpha_i x_i = \frac 1 \beta x~,
  \end{align*}
  and by convexity, we conclude $x/\beta \in \pi(C)$.
  Finally, if $g(x) = 0$, then $g(x/\beta) = 0$, but as $x/\beta \in \pi(C)$, we must have $x/\beta = 0$, whence $x=0$.
  As $x\in\V/\spn S$ was arbitrary, we conclude $\ker g = \{0\}$.
\end{proof}

  \begin{proof}[of Corollary~\ref{cor:bayes-risk-ident-lower-bound}]

  For the upper bound, $\Gamma\in\ID(\P)$ implies $(\lbar,\Gamma)\in\ID(\P)$, as if $V(a,y)$ identifies $\Gamma$, then $V'((r,a),y) = \bigl(L(a,y) - r,\, V(a,y)\bigr)$ identifies $(\lbar,\Gamma)$.
  Corollary~\ref{cor:elic-complex-bayes-risk} then gives $\elici(\lbar) \leq k+1$.
  For the lower bounds,
  Theorem~\ref{thm:bayes-risk-lower-bound} gives $\elici(\lbar) \geq k$ with equality if $\lbar$ is a link of $\Gamma$.

  For the stronger lower bound of $k+1$, let $V$ and $r$ be the identification function and report from Condition~\ref{cond:v-interior}, and assume $\lbar$ is non-constant on $\Gamma_r$.
  Given $\hat\Gamma:\P\to\reals^\ell$ and $g$ from Theorem~\ref{thm:bayes-risk-lower-bound}, so that $\hat\Gamma$ is elicitable and identifiable and $\lbar = g \circ \hat\Gamma$, we wish to show $\ell \geq k+1$.
  By the proof of Theorem~\ref{thm:bayes-risk-lower-bound}, $\hat\Gamma$ refines $\Gamma$, and moreover $\ell\geq k$.

  Now suppose $\ell = k$ for a contradiction.
  By the proof of Lemma~\ref{lem:lin-alg-nested-level-sets}, there is a distribution $p\in\Gamma_r$ such that if $p \in \hat \Gamma_{\hat r} \subseteq \Gamma_r$, which is guaranteed by refinement, then $0\in\interior\{\hat V(\hat r,p) : p\in\P\}$.
  Applying Lemma~\ref{lem:lin-alg-span} to the function $f:\spn \P \to \reals^k$, $q \mapsto V(r,q)$, we have $\spn \ker f = \spn \Gamma_r$.
  Applying Lemma~\ref{lem:lin-alg-span} again to $\hat f:\spn \P \to \reals^k$, $q \mapsto \hat V(\hat r,q)$, we have $\spn \ker \hat f = \spn \hat\Gamma_{\hat r}$.
  As $\hat\Gamma_{\hat r} \subseteq \Gamma_r$, we have $\ker \hat f = \spn \hat\Gamma_{\hat r} \subseteq \spn \Gamma_r = \ker f$.
  By the first isomorphism theorem, we also have $\codim \ker \hat f = \codim \ker f = k$, as the images of these linear maps span all of $\reals^k$.
  By the third isomorphism theorem we conclude $\Gamma_r = \hat\Gamma_{\hat r}$.
  Since by assumption $\lbar$ is non-constant on $\Gamma_r = \hat\Gamma_{\hat r}$, we have distributions $p,p' \in\hat\Gamma_{\hat r}$ with $\lbar(p)\neq\lbar(p')$, which contradicts $\lbar$ being a link of $\hat\Gamma$: $\lbar(p) = g(\hat\Gamma(p)) = g(\hat r) = g(\hat\Gamma(p')) = \lbar(p')$.
\end{proof}

\subsection{Bounds for $\Cstrict$ and $\Cstrong$}
\label{sec:omitted-proofs-other-classes}

We now give the full details of our upper and lower bounds for strictly and strongly convex losses.
Examining the form $L^*((r,a),y) = H(r) + h(r)(L(a,y)-r)$ from eq.~\eqref{eq:elic-bayes-risk-cor}, which established the main upper bound, we see that as long as $h$ does not decrease ``too quickly'' relative to the curvature of $L$, the loss $L^*((r,a),y)$ is still strictly convex in $(r,a)$.
\begin{proposition}
  \label{prop:cvx-elic}
  Let $\Gamma \in \Cstrict$, $\Gamma:\P\to\reals^k$, $k\in\N$, be elicited by a twice-differentiable, strictly convex, bounded loss function $L$.
  If $\Gamma$ satisfies Condition~\ref{cond:v-interior} for some $r\in\Gamma(\P)$, and $\lbar$ is non-constant on $\Gamma_r$, and there exists $\alpha>0$ with
  \begin{equation}
    \label{eq:cvx-elic-cond}
    \forall y\in\Y,\;\;
    \alpha \nabla_a^2 L(\cdot,y) \succ \nabla_a L(\cdot,y)\nabla_a L(\cdot,y)^\tr~,
  \end{equation}
  then $\elic_\Cstrict(\lbar) = k+1$.
\end{proposition}

\begin{proof}
  For the lower bound,  the conditions of the Proposition allow us to apply, Corollary~\ref{cor:bayes-risk-ident-lower-bound}, which gives us $\elici(\lbar) = k+1$.
  By Proposition~\ref{prop:nested-convex-classes-long}, we conclude $\elic_{\Cstrict}(\lbar) \geq k+1$.
  
  For the upper bound, let $L\in[0,B]$ without loss of generality, so that $\lbar\in[0,B]$.
  The pair $(\lbar,\Gamma)$ is bounded.
  Take $h(r) = \alpha + B - r$, the $L^*((r,a),y)$ we obtain from Corollary~\ref{cor:elic-complex-bayes-risk}, eq.~\eqref{eq:elic-bayes-risk}, is given by
  \begin{equation}
    \label{eq:cvx-elic-lstar}
    L^*((r,a),y) = \frac {r^2} 2 + (\alpha + B - r) L(a,y)~.
  \end{equation}
  As $L$ is twice differentiable, we may verify the strict convexity of $L^*$ by checking that its Hessian is positive definite,
  \begin{equation}
    \label{eq:2}
    \nabla_{(r,a)}^2 L^*(\cdot,y) =
    \begin{bmatrix}
      1 & -\nabla_a L(\cdot,y) \\
      -\nabla_a L(\cdot,y) & (\alpha + B - r)\nabla_a^2 L(\cdot,y)
    \end{bmatrix}~.
  \end{equation}
  By the Schur complement theorem, $\nabla_{(r,a)}^2 L^*(\cdot,y)$ is positive definite if any only if
  \begin{equation}
    \label{eq:1}
    (\alpha + B - r)\nabla_a^2 L(\cdot,y) - (-\nabla_a L(\cdot,y))(1)^{-1}(-\nabla_a L(\cdot,y))^\tr \succ 0~,
  \end{equation}
  which is implied by the condition~\eqref{eq:cvx-elic-cond} as $B - r \geq 0$ and thus $(B-r)\nabla_a^2 L(\cdot,y)\succeq 0$.
  Moreover, Lipschitz continuity and differentiability of $L$ implies the same of $L^*$.
  We have now shown $(\lbar,\Gamma) \in \Cstrict$, giving the result.
\end{proof}

Intuitively, Proposition~\ref{prop:cvx-elic} tells us that as long as $L$ is ``convex enough'', its curvature is sufficient to offset the decreasing effect of the $h(r)$ coefficient in eq.~\eqref{eq:loss-elic-minumum} and~\eqref{eq:elic-bayes-risk-cor}.
Naturally, then, this result gives a bound for strongly convex losses $L$ as well.
The Hessian of a strongly convex $L$ satisfies $\nabla_a^2 L(\cdot,y) \succeq \mu I$ for some $\mu>0$.
Thus, letting $\lambda$ be the supremum of largest eigenvalue of $\nabla_a L(\cdot,y)\nabla_a L(\cdot,y)^\tr$ over all $a$, which is finite by boundedness of $L$ and compactness of the range of $\Gamma$, we can simply take $\alpha = 2\lambda/\mu$ and proceed as in Proposition~\ref{prop:cvx-elic}.
We instead use a different proof technique, which allows us to lift the twice-differentiability assumption as well.

\begin{proof}[of Proposition~\ref{prop:str-cvx-elic}]
  As in Proposition~\ref{prop:cvx-elic}, our conditions together with Corollary~\ref{cor:bayes-risk-ident-lower-bound} and Proposition~\ref{prop:nested-convex-classes-long} give the lower bound.
  For the upper bound, fix the outcome $y\in\Y$ and let $F(a) := L(a,y)$.
  We have by assumption that $L$, and thus $F$, is $\mu$-strongly convex for some $\mu>0$.
  Taking $L^*$ in eq.~\eqref{eq:cvx-elic-lstar}, and letting $C=\alpha + B$, we have
  \begin{align*}
    &L^*((r,a),y) - L^*((s,b),y) - \nabla_{(s,b)} L^*((s,b),y)
    \\
    &= \tfrac 1 2 r^2 + (C-r) F(a) - \tfrac 1 2 s^2 - (C-s) F(b) - \bigl((s-F(b))(r-s)
    + (C-s)\nabla F(b) \cdot (a-b)\bigr)
    \\
    &= \tfrac 1 2 (r-s)^2 + (C-r) F(a) - (C-s) F(b) + F(b)(r-s)
    - (C-s)\nabla F(b) \cdot (a-b)
    \\
    &= \tfrac 1 2 (r-s)^2 + (C-r) (F(a) - F(b)) - (C-s)\nabla F(b) \cdot (a-b)
    \\
    &\geq \tfrac 1 2 (r-s)^2 + (C-r) \tfrac \mu 2 \|a-b\|^2 + (s-r)\nabla F(b) \cdot (a-b)
    \\
    &\geq \tfrac 1 2 (r-s)^2 + (C-B) \tfrac \mu 2 \|a-b\|^2 - |s-r| \|\nabla F(b)\| \|a-b\|~.
  \end{align*}
  Let $\nabla_{\max} = \sup_{a\in \Gamma(\P),y\in\Y} \|\nabla L(\cdot,y)\|$ be the largest gradient magnitude of $L$, which is finite by boundedness of $L$ and compactness of the range of $\Gamma$.
  Letting $C = (8\nabla_{\max}^2 + \tfrac 1 2)/\mu + B$, we have
  \begin{align*}
    &L^*((r,a),y) - L^*((s,b),y) - \nabla_{(s,b)} L^*((s,b),y)
    \\
    &\geq \tfrac 1 2 (r-s)^2 + (4\nabla_{\max}^2 + \tfrac 1 4) \|a-b\|^2 - \nabla_{\max} |s-r| \|a-b\|
    \\
    &= \tfrac 1 4 (r-s)^2 + \tfrac 1 4 \|a-b\|^2 + \left(\tfrac 1 2 |r-s| - 2 \nabla_{\max} \|a-b\|\right)^2~,
  \end{align*}
  which as the third term is nonnegative, shows $L^*$ to be $\tfrac 1 2$-strongly convex.
\end{proof}

\section{Additional Discussion}
\label{sec:additional}

\subsection{Comparison to Other Definitions in the Literature}
\label{sec:comp-other-defin}

The literature has seen several variations on the definition of elicitation complexity, which fall into three broad categories:
(1) different choices of the class $\C$,
(2) variations on the type of loss function allowed,
and (3) different requirements on the link function $f$. 
For the most part, all of these restrictions can be cast as restrictions on $\C$; for example, we imposed restrictions on the loss for our classes $\Cstrict$ and $\Cstrong$ by restricting to properties elicited by such a loss.

\paragraph{1. Class of properties $\C$}
Choices of $\C$ in the literature include continuous properties \citep{steinwart2014elicitation},  linear properties or expectations \citep{agarwal2015consistent}, and properties whose components are themselves elicitable \citep{lambert2008eliciting} meaning every $\hat\Gamma\in\C$, $\hat\Gamma:\P\to\reals^k$, should have $(\hat\Gamma)_i$ be elicitable for $i=1,\ldots,k$.
Our classes $\I$ and $\Clin$ readily fall into this category.

\paragraph{2. Varying the loss function}
Some classes of properties are naturally defined by restrictions on the loss function.
For example, for efficient optimization in empirical risk minimization, it is common to restrict to convex losses; elicitation complexity with respect to the class of convex-elicitable properties, $\Ccvx$, is closely related to the notion of \myemph{convex calibration dimension} \citep{ramaswamy2016convex}.
The classes $\Cstrict$ and $\Cstrong$ impose further restrictions on the loss so that the resulting properties are identifiable.
The restriction of Lambert et al.~\citeyearpar{lambert2008eliciting} that properties have elicitable components could also be cast as the restriction that the loss function $L$ be \myemph{separable}, meaning $L(r,y) = \sum_{i=1}^k L(r_i,y)$ where $r\in\reals^k$.
Another recent variation is a multi-observation loss function, which is allowed multiple independent realizations of the random variable $Y$, that is, takes the form $L(r,y_1,\ldots,y_m)$ \citep{casalaina-martin2017multi,frongillo2019multi}.
Multi-observation losses can reduce elicitation complexity, sometimes dramatically: the variance and 2-norm both have complexity 1 with respect to $m$-observation losses even for $m = 2$, compared to complexities 2 and $|\Y|-1$ for the usual setting of $m=1$ (\S~\ref{sec:ex-cvx-ftns-of-means}).

\paragraph{3. Restricting the link function}
Fissler and Ziegel~\citeyearpar{fissler2016higher} propose a definition of complexity as the smallest $k$ such that $\Gamma$ is a component of a $k$-dimensional elicitable property.
Equivalenttly, the link function must take the simple form $f(r) = r_1$, the first component of $r\in\reals^k$.
It would also be natural to restrict to the broader class of continuous, or differentiable, link functions.

In general, we believe that the notion of elicitation complexity put forth in this paper, Definition~\ref{def:elic-complex}, is best suited to studying the difficulty of eliciting properties:
viewing $f$ as a potentially dimension-reducing link function, our definition captures the minimum number of dimensions needed in a point estimation or empirical risk minimization for the property in question, followed by a simple one-time application of $f$.
To further justify this claim, we observe that our definition yields weakly lower complexity than most of the other common definitions, and strictly lower for some natural properties.
For example, consider the first-component-link definition of Fissler and Ziegel~\citeyearpar{fissler2016higher} for the squared mean $\Gamma(p) = \E_p[Y]^2$ when $\Y=\reals$.
As we saw in Remark~\ref{remark:mean-squared}, this property has $\elic_\Clin(\Gamma)=1$, yet as it is not directly elicitable, it has complexity 2 under the Fissler--Ziegel definition.
This complexity is achieved, for example, via the property $\hat\Gamma(p) = (\E_p[Y]^2,\E_p[Y])$, elicited by $L(r,y) = (r_2-y)^2 + \ones\{r_1\neq r_2^2\}$.
Moreover, in light of Remark~\ref{remark:everything-1-elic}, without further restrictions on the loss, link, or class of properties, essentially all properties have complexity 2 under their definition.
Finally, under the component-wise-elicitable $\C$ of Lambert et al.~\citeyearpar{lambert2008eliciting}, the property $\Gamma(p) = \max_{y\in\Y} p(\{y\})$ for finite $\Y$ has complexity $\elic_\C(\Gamma) = |\Y|-1$, whereas we show in \S~\ref{sec:ex-modal-mass} that $\elicifin(\Gamma) = 2$, where $\I^\fin$ is a slight generalization of identifiability to allow for finite properties, i.e., $\Gamma:\P\to\R$ for a finite set $\R$.

\subsection{Elicitation Versus Elicitation Complexity}
\label{sec:elic-vers-elic}

In Remark~\ref{remark:mean-squared}, we gave the example of the squared mean to illustrate the fact that non-elicitable properties can still have elicitation complexity 1.
Here we discuss the converse: elicitability does not necessarily imply elicitation complexity 1.
In fact, recent results on the elicitation complexity of the mode and modal interval give a natural example of a property which is clearly elicitable but whose elicitation complexity with respect to identifiable properties is infinite.

While the mode of a distribution is elicitable when $\Y$ is a finite set, e.g.\ via 0-1 loss, it was recently shown that, for sufficiently rich choices of distributions $\P$, the mode is not elicitable over $\Y=\reals$ \citep{heinrich2013mode}.
In this real-valued setting, the mode is defined for e.g.\ continuous densities as the $\argmax$ of the density value.
To circumvent this impossibility, in practice one may attempt to ``approximate'' the mode as follows.
Given a parameter $\beta$, we define the midpoint of a modal interval as the property $\mi_\beta(p) = \argmax_{a\in\reals} p([a-\beta,a+\beta])$, namely, the midpoint of the interval of width $2\beta$ with the maximum probability mass.
The modal interval is directly elicitable, via the loss $L(r,y) = \ones\{|r-y|>\beta\}$.

Interestingly, recent work shows that the mode and modal interval both have infinite elicitation complexity with respect to $\I$, despite the latter being directly elicitable \citep{dearborn2019indirect}.
In other words, while $\mi_\beta$ is directly elicitable, we have $\elici(\mode) = \elici(\mi_\beta) = \infty$.
We may also define the modal mass of width $\beta$ as the mass of the modal interval, $\gamma_\beta(p) = \max_{a\in\reals} p([a-\beta,a+\beta])$, which satisfies $\gamma_\beta(p) = 1 - \lbar(p)$; a modification of Theorem~\ref{thm:bayes-risk-lower-bound} allowing for non-convex $\P$ (cf.\ \S~\ref{sec:extens-non-conv}) would thus conclude $\elici(\gamma_\beta) = \infty$ as well.
These lower bounds may appear to contradict results showing that real-valued properties are elicitable if and only if they are identifiable \citep{lambert2018elicitation,steinwart2014elicitation}, but such results require the property in question to be continuous, and $\mi_\beta$ is not.
Intuitively, these lower bounds suggest that the loss $L$ eliciting $\mi_\beta$ will be hard to optimize in practice.
More generally, some restriction $\C$ on the properties and losses in question may be sensible even in a direct elicitation setting.

\subsection{Mode and Modal Mass for Finite $\Y$}
\label{sec:ex-modal-mass}

In the case of finite $\Y$, with $\P$ taken to be the probability simplex, all distributions over~$\Y$, we define the modal mass $\gamma(p) = \max_{y\in\Y} p(\{y\})$ as the highest probability assigned to any outcome.
In other words, the modal mass is the probability assigned to the mode of $p$, defined by $\mode(p) = \argmax_{y\in\Y} p(\{y\})$,
which we note is set-valued in general.
The mode of $p$ is elicitable via 0-1 loss $L(r,y) = \ones\{r\neq y\}$.
Here $\ones$ denotes the indicator function.
The modal mass is not elicitable, however, as evidenced by its nonconvex level sets, and hence we turn to its elicitation complexity.

The form of the modal mass is reminscent of eq.~\eqref{eq:prop-minimum} from Theorem~\ref{thm:elic-minimum}, and indeed $-\gamma$ is the minimum expected value over $X_a(y) = -\ones\{y=a\}$.
More directly, we can see that $\gamma$ is $1$ minus the Bayes risk of 0-1 loss: $\gamma(p) = \max_{r\in\Y} \E_p \ones\{r = y\} = 1 - \min_{r\in\Y} \E_p \ones\{r \neq y\} = 1-\lbar(p)$.
Unfortunately, we cannot immediately apply Corollary~\ref{cor:bayes-risk-ident-lower-bound}, as the mode is not identifiable.
Indeed, no nonconstant finite property is identifiable; see e.g.\ \citet[Lemma 2.4]{fissler2019order}.

To work around this technical barrier while keeping with the spirit of our framework, we replace $\I$ with the class $\I^\fin$ of of properties which are identifiable after possibly conditioning on some elicitable finite property.
Nonconstant finite elicitable properties are necessarily set-valued, though unless they are redundant, only on a subset of the boundary of each level set $\Gamma'_{r'}$ \citep{lambert2009eliciting}.
Formally, let us define $\I_k'$ to be the class of properties $\Gamma = (\Gamma',\Gamma'')$ where $\Gamma' : \P \to 2^{\R'}$ is elicitable with $|\R'|<\infty$, and $\Gamma'' \in \I_{k-1}(\Gamma'_{r'})$ for all $r' \in \R'$, where $\Gamma'_{r'} = \{p\in\P : \Gamma'(p)=r'\}$ is a level set of $\Gamma'$.
For the case $k=1$, we take $\ID_0$ to be the set of constant properties.
That is, $\hat\Gamma$ is a product of an elicitable finite property, and a vector-valued, or real-valued, property which is identifiable conditioned on that finite property.
We then take $\I^\fin_k = \I'_k \cup \I_k$, and $\I^\fin = \bigcup_k \I^\fin_k$.
Any elicitable finite property $\Gamma$ is trivially in $\I^\fin_1$, by taking $\Gamma'=\Gamma$.

With this formalism in hand, we see that $\Gamma = (\mode,\gamma) \in \I^\fin_2$:
as observed above, $\mode(\cdot)$ is finite and elicitable, and for all $a\in\R'=\Y$ the function $V_a(r,y) = \ones\{y = a\} - r$ identifies $\gamma$ conditioned on $a \in \mode(\cdot)$, that is, when restricting to the distributions with mode $a$.
Theorem~\ref{thm:elic-minimum} now applies to give $\elicifin(\gamma)\leq 2$.
In fact, for $\P$ sufficiently rich, $\elicifin(\gamma)=2$ as $\gamma$ is not a link of a finite property.
As discussed in \S~\ref{sec:comp-other-defin}, this low complexity gives an interesting contrast to the component-wise-elicitable $\C$ of Lambert et al.~\citeyearpar{lambert2008eliciting}, where $\elic_\C(\gamma) = |\Y|-1$, the maximum possible complexity.

\subsection{Illustration of the Difference Between Joint Elicitability and Conditional Elicitability}

\begin{figure}[!h]
  \centering
  \includegraphics[width=0.4\textwidth]{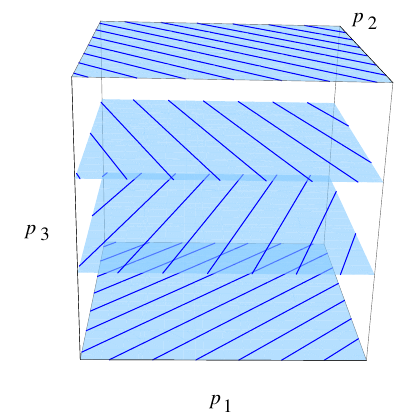}
  \includegraphics[width=0.4\textwidth]{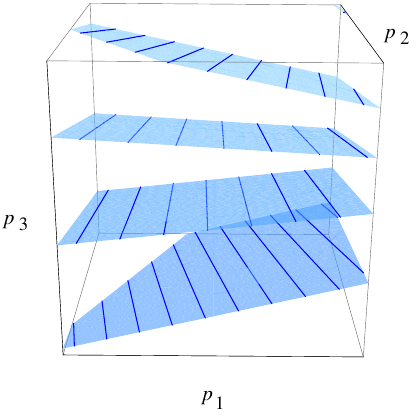}
  \caption{Depictions of the level sets of two properties on outcomes $\Y=\{1,2,3,4\}$, one elicitable and the other not, at least not by a twice differentiable loss function.
    \textbf{Left:} The property depicted is $\Gamma(p) = (p_3, p_1 + p_2 p_3)$, an example taken from \citet{frongillo2015vector} which is shown not to be elicitable by any twice differentiable loss function.
    \textbf{Right:} Let $\gamma(p)$ be implicitly defined as the solution $r$ to the equation $\tfrac 1 3 \sin(r) p_1 + \tfrac 1 4 \cos(r) p_2 + p_3 = r$.
    One can check that the loss $L(r,y) = 6 r^2 + 4 \cos(r) \ones\{y=1\} - 3 \sin(r) \ones\{y=2\} - 12 r \ones\{y=3\}$ elicits $\gamma$.
    The property depicted is $\Gamma = (\lbar,\gamma)$, which is elicitable by Theorem~\ref{thm:elic-minimum}.
    \\[4pt]
    Interestingly, both properties are conditionally elicitable, conditioned on $\Gamma'(p) = p_3 = \E_p[\ones\{Y\!=\!3\}]$, as illustrated by the planes: the height of the plane, the intersept $(p_3,0,0)$ for example, is elicitable as an expected value, and conditioned on this plane, the properties are both linear and thus links of expected values, which are also elicitable.}
  \label{fig:3d-examples}
\end{figure}

Figure~\ref{fig:3d-examples} gives an example of two properties that are both conditionally elicitable, but one is elicitable while the other is not.
It illustrates the subtlety of characterizing all elicitable vector-valued properties, perhaps the most fundamental open question in this literature.
Indeed, even nontrivial examples of vector-valued properties which were not simply a vector of real-valued elicitable properties, or a link of such a vector of properties, were sparse before Theorem~\ref{thm:elic-minimum}.
It may be that some crucial insight lies in the difference between the seemingly similar properties in Figure~\ref{fig:3d-examples}, of which one is elicitable and the other is not.
One interesting question toward this general characterization is the following: do there exist elicitable properties which are not links of properties having at least one elicitable component?
This question is trivial without allowing for link functions, due to examples such as $\Gamma(Y) = (\E[Y] + \Var[Y], \E[Y] - \Var[Y])$ where we take an elicitable property $(\E[Y],\Var[Y])$ and apply an invertible link to disrupt the elicitability of each component.

\subsection{Extensions to Non-convex $\P$}
\label{sec:extens-non-conv}

Throughout the paper we have assumed the set of probability measures $\P$ is convex.
This assumption is primarily for ease of exposition; here we briefly discuss which results may be extended to non-convex $\P$.
First, our elicitation complexity upper bounds all hold for any class $\P$ for which the relevant expectations, e.g.\ of the loss, are finite, and in particular, do not require $\P$ to be convex.
Our lower bounds are more delicate.
The proof of Lemma~\ref{lem:elic-complex-bayes-concave}, which shows a certain strict concavity property for Bayes risks, does rely on $\P$ being convex.
Nonetheless, this restriction is not absolutely necessary; as we discuss in \S~\ref{sec:proof-lemma-bayes-concave}, for non-convex $\P$, the inequality in Lemma~\ref{lem:elic-complex-bayes-concave} would simply hold for any $\lambda \in (0,1)$ such that $\lambda p + (1-\lambda) p' \in \P$.
As such, with care, our lower bounds could be extended to non-convex $\P$ when one can still guarantee a version of Lemma~\ref{lem:main-iden-lower-bound} and the technical lemmas supporting Corollary~\ref{cor:bayes-risk-ident-lower-bound}.

\end{document}